\documentclass[DTMColor]{ROB-New}

\usepackage[linesnumbered,ruled,vlined]{algorithm2e}
\SetKwComment{Comment}{/* }{ */}
\RestyleAlgo{ruled}

\usepackage{caption}
\usepackage{subcaption}
\usepackage{enumitem}

\usepackage{array}
\newcommand{\PreserveBackslash}[1]{\let\temp=\\#1\let\\=\temp}
\newcolumntype{C}[1]{>{\PreserveBackslash\centering}p{#1}}

\begin{document}

\authormark{Duy Nam Bui and Manh Duong Phung}

\articletype{RESEARCH ARTICLE}

% \jnlPage{1}{7}
% \jyear{2023}
% \jdoi{10.1017/xxxxx}

\title{Radial Basis Function Neural Networks for Formation Control of Unmanned Aerial Vehicles}

\author[1]{Duy-Nam Bui}
\address[1]{Vietnam National University, Hanoi, Vietnam}

\author[2]{Manh Duong Phung\hyperlink{corr}{*}}
\address[2]{Fulbright University Vietnam, Ho Chi Minh City, Vietnam}
\address{\hypertarget{corr}{*}Corresponding author. \email{duong.phung@fulbright.edu.vn}}

% \received{xx xxx xxx}
% \revised{xx xxx xxx}
% \accepted{xx xxx xxx}


\abstract{ This paper addresses the problem of controlling multiple unmanned aerial vehicles (UAVs) cooperating in a formation to carry out a complex task such as surface inspection. We first use the virtual leader-follower model to determine the topology and trajectory of the formation. A double-loop control system combining backstepping and sliding mode control techniques is then designed for the UAVs to track the trajectory. A radial basis function neural network (RBFNN) capable of estimating external disturbances is developed to enhance the robustness of the controller. The stability of the controller is proven by using the Lyapunov theorem. A number of comparisons and software-in-the-loop (SIL) tests have been conducted to evaluate the performance of the proposed controller. The results show that our controller not only outperforms other state-of-the-art controllers but is also sufficient for complex tasks of UAVs such as collecting surface data for inspection. The source code of our controller can be found at {\fontfamily{pcr}\selectfont \url{https://github.com/duynamrcv/rbf_bsmc}}.
}

\maketitle

\section{Introduction}
Unmanned aerial vehicles, when combined with computer vision technologies, can collect visual data of structures to provide valuable information for various tasks such as inspecting structural surfaces \cite{zeng_zhong2023,rizia_reyes2022},  reconstructing 3D models \cite{INZERILLO2018457, ZHAO2021103832}, identifying cracks \cite{CHEN2021102913,PENG2021123896}, and detecting corrosion and rust on steel bridges \cite{la_dinh__2019,TIAN2022104043}. However, using a single UAV for these tasks is inefficient due to the large size of the structures and the limited battery capacity of the UAV. A group of UAVs flying in a formation can be used to overcome those limitations \cite{9341089,9384182,8593930,OH2015424}. The formation allows the UAVs to perform collaborative inspection to increase the efficiency and accuracy of data collection. It also allows for safe operation, as the formation control can prevent collision among the UAVs.

In formation control, the leader-follower approach is commonly used to provide flexibility in topology and trajectory selection\cite{Liu2018,8756125}. In the standard leader-follower method, one UAV is assigned as the leader, and the others are followers. The leader plays the role of a reference node for the followers to determine their locations to form the desired topology. The limitation of this approach, however, is the dependence of the system on the leader. If the leader is malfunctioning, the whole system will fail. The virtual leader-follower model can be used to cope with this problem. In this approach, the leader is purely a virtual entity, serving as a reference point for the followers to determine their positions \cite{ZHENG2021389}. By decoupling the physical leader from the model, this method mitigates the risk of complete system failure.

In the leader-follower model, linear controllers are commonly used to control individual UAVs to form the desired topology  \cite{ZHENG2021389,Rinaldi2013,doi:10.1177/0278364909104290}. In \cite{Rinaldi2013}, linear quadratic and neural networks-based controllers are combined to control a group of UAVs considering their full dynamics. In \cite{doi:10.1177/0278364909104290}, the receding horizon control is employed to yield a fast convergence rate of the formation tracking control. This controller also considers the orientation between the leader and the followers for accurate formation. The decentralized $H_\infty$-PID controller is introduced in \cite{9908553} to maneuver a group of UAVs to deal with the external disturbance and trailing vortex coupling from their neighbor UAVs. A leader-follower formation control technique is presented in \mbox{\cite{Chen2023}} to address issues related to backward error and suboptimal dynamic speed tracking in PID neural network control. Linear controllers, however, have limitations in handling constraints and parameter variation, especially when applied to nonlinear systems like UAVs.

In another approach, nonlinear controllers have been used for formation control \cite{Fahimi2008,KeymasiKhalaji2019,DEHGHANI2016318}. In \cite{4601469}, first and second-order sliding-mode controllers are deployed to assure the asymptotic stability of the formation, taking into account modeling uncertainties. In \cite{7904763}, an adaptive controller using the dynamic estimation of the distance between the leader and the followers is introduced to address uncertainties related to positioning errors. Two finite-time observers are used in \cite{HUANG20204034} to deal with bounded external disturbance force and torque. In \mbox{\cite{9707476}}, a non-uniform vector field that dynamically varies in magnitude and direction is employed to deal with the influence of wind in UAV formation control. A distributed model predictive control algorithm is introduced in \mbox{\cite{10064191}} to coordinate the operation of a fleet of UAVs considering their spatial kinematics and unidirectional data transmissions. However, the convergence of these controllers depends on the characteristics of disturbances, which are hard to model due to their varying nature. A sufficient approach would be utilizing neural networks such as the radial basis function neural network (RBFNN) to estimate disturbances and use it as the feedback for control \cite{YANG2021243,SHOJAEI2016372,8651430, 6796390}.

% \begin{figure}
%     \centering
%     \includegraphics[width=0.8\textwidth]{images/gazebo.png}
%     \caption{The flight trajectories of the formation (3D view)}
%     \label{fig:gazebo_3d}
% \end{figure}

In this work, we present a new controller for a group of UAVs cooperating in a formation. The UAVs use the virtual leader-follower model to determine their trajectory and form the desired topology. The controller is developed using the backstepping and sliding mode control techniques. An RBFNN is then introduced to estimate external disturbances for better control performance. Our contributions in this work are as follows:
% (i) proposal of a dual-loop controller using nonlinear techniques \hl{namely backstepping sliding mode control (BSMC)} that manipulate both the position and attitude of the UAVs to form the desired topology and track the reference trajectory; (ii) development of an RBFNN to estimate external disturbances and use it as the feedback for better tracking performance; and (iii) derivation of the stability proof for the designed controller using Lyapunov's theorem, which is essential to ensure stable operation of the UAVs in various conditions.

\begin{enumerate}[label=\roman*.]
    \item The proposal of a new controller for UAV formation that is constructed by combining backstepping and sliding mode control techniques, thereby enabling the elimination of nonlinear components and enhancing system robustness. Additionally, the adverse effects associated with these controllers, such as ``explosion of term'' and ``chattering'', are mitigated through the approximation of unknown factors by the neural network. As the result, the developed controller not only addresses the drawbacks of the aforementioned techniques but also augments the adaptability of the UAV system.
    \item The design of a radial basis function neural network (RBFNN) that is capable of estimating external disturbances to compensate for input force control signals, thereby enabling the controller to maintain the required control quality.
    \item The derivation of the stability proof for the designed controller using Lyapunov’s theorem, which is essential to ensure stable operation of the UAVs under conditions affected by external forces.
    \item The comparison of the proposed controller with other popular methods including model predictive control (MPC), backstepping sliding mode control (BSMC) and sliding mode control (SMC) in different scenarios to confirm its superior performance. Software-in-the-loop tests were also conducted with a cooperative bridge inspection task to verify the validity of the proposed method for practical applications.
\end{enumerate}

The rest of this paper is structured as follows. Section \ref{system} presents the dynamic and formation models of the UAVs. Section \ref{controller} introduces the proposed controller. Section \ref{result} shows evaluation results. The paper ends with conclusions described in \ref{con}.
\section{Problem formulation}\label{system}
\begin{figure}
    \centering
    \includegraphics[width=0.45\textwidth]{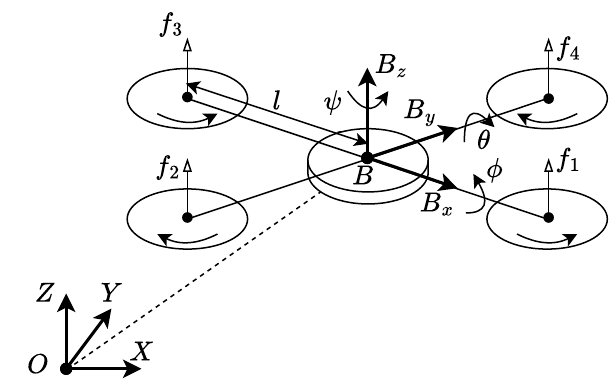}
    \caption{The structure of the quadrotor UAV in the global frame}
    \label{fig:model}
\end{figure}
To control a group of UAVs, we first consider their dynamic model and formation topology with details as follows.

\subsection{UAV dynamic model}

Consider a group of $n$ UAVs, each is a quadrotor with two pairs of propellers rotating in opposite directions, as described in Figure \ref{fig:model}. Frames $BB_xB_yB_z$ and $OXYZ$ are respectively the body-fixed and inertial frames. We use Euler angles to represent the attitude of the UAV. The configuration of the UAV includes its position $\xi=\left[x,y,z\right]^T$ and Euler angles $\Xi=\left[\phi,\theta,\psi\right]^T$, with $\left\vert\phi\right\vert\leq\pi/2$, $\left\vert\theta\right\vert\leq\pi/2$ and $\left\vert\psi\right\vert\leq\pi$. Those angles represent the roll, pitch and yaw orientation of the UAV, respectively. Control signals of the UAV are defined as follows:
\begin{equation}
    \left[\begin{array}{c}
    f_{t}\\
    \tau_{\phi }\\
    \tau_{\theta }\\
    \tau_{\psi }
    \end{array}\right]=\left[\begin{array}{c}
    f_{1}+f_{2}+f_{3}+f_{4}\\
    l\left(f_{4}-f_{2}\right)\\
    l\left(f_{3}-f_{1}\right)\\
    \tau_{2}+\tau_{4}-\tau_{1}-\tau_{3}
    \end{array}\right],
\end{equation}
where $l$ is the arm length; $f_{t}$ is the total thrust of four propellers; $\tau_\phi$, $\tau_\theta$, $\tau_\psi$ are the torques in three axes; and $f_i$ and $\tau_i$, with $i=\left\{1,2,3,4\right\}$, are the forces and torques generated by four propellers, respectively. According to \cite{Furrer2016}, the dynamic model of the UAV is described as follows:

\begin{equation}
    \begin{aligned}
        \ddot{x}&=\left(\cos\phi\sin\theta\cos\psi+\sin\phi\sin\psi\right)\dfrac{f_{t}}{m}+\dfrac{d_x}{m}\\
        \ddot{y}&=\left(\cos\phi\sin\theta\sin\psi-\sin\phi\cos\psi\right)\dfrac{f_{t}}{m}+\dfrac{d_y}{m}\\
        \ddot{z}&=\cos\phi\cos\theta\dfrac{f_{t}}{m}-g+\dfrac{d_z}{m}\\
        \ddot{\phi}&=\dfrac{\dot{\theta}\dot{\psi}\left(I_{y}-I_{z}\right)+\tau_{\phi}}{I_{x}}\\
        \ddot{\theta}&=\dfrac{\dot{\phi}\dot{\psi}\left(I_{z}-I_{x}\right)+\tau_{\theta}}{I_{y}}\\
        \ddot{\psi}&=\dfrac{\dot{\phi}\dot{\theta}\left(I_{x}-I_{y}\right)+\tau_{\psi}}{I_{z}}
    \end{aligned}
    \label{eqn:dynamic}
\end{equation}
where $I_x$, $I_y$, $I_z$ are the moments of inertia, $m$ is the mass of the UAV, $g$ is the gravitational acceleration, and $\left[d_x,d_y,d_z\right]^T$ is the disturbance caused by factors such as wind or turbulent flows.

\subsection{UAV formation model}
The formation model used in this work is the virtual leader-follower model with two main components:
\begin{itemize}
    \item Virtual leader: a virtual leader is a non-physical UAV used as a reference for other UAVs to determine their position. Its trajectory represents the trajectory of the UAV group. 
    \item Follower: a follower is a UAV that adjusts its position based on the virtual leader. Given the reference trajectory of the leader and the expected topology, the followers calculate their trajectories and then track them to form the desired formation.
\end{itemize}

\begin{figure}
    \centering
    \includegraphics[width=0.45\textwidth]{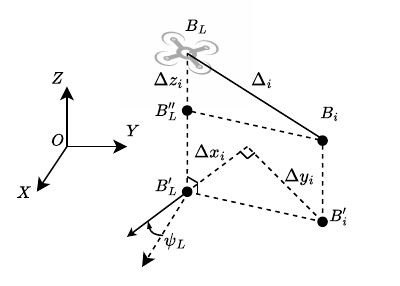}
    \caption{Illustration of the virtual leader-follower formation structure}
    \label{fig:formation}
\end{figure}

Consider virtual leader $B_L$ having position $\xi_L = \left[x_L,y_L,z_L\right]^T$ and heading angle $\psi_L$ and follower $B_i$ having position $\xi_L = \left[x_{i},y_{i},z_{i}\right]^T$ and yaw angle $\psi_{i}$. Let $B'_L$ and $B'_i$ be their projection on the $OXY$ plane, respectively, $B''_L$ be the projection of $B_i$ on $B_LB'_L$, and $\Delta_i=\left[\Delta x_{i}, \Delta y_{i}, \Delta z_{i}\right]^T$ be the desired distance between follower $B_i$ and the virtual leader, as depicted in Figure \ref{fig:formation}. Since $\Delta z_i=B_LB''_L$, the desired position of follower $B_i$ can be computed as:
\begin{equation}
    \begin{aligned}
        % \left[\begin{array}{c}
        % x_{i}^d\\
        % y_{i}^d\\
        % z_{i}^d
        % \end{array}\right]&=\left[\begin{array}{ccc}
        % \cos\psi_{L} & -\sin\psi_{L} & 0\\
        % \sin\psi_{L} & \cos\psi_{L} & 0\\
        % 0 & 0 & 1
        % \end{array}\right]\left[\begin{array}{c}
        % \Delta x_{i}\\
        % \Delta y_{i}\\
        % \Delta z_{i}
        % \end{array}\right]+\left[\begin{array}{c}
        % x_{L}\\
        % y_{L}\\
        % z_{L}
        % \end{array}\right],\\
        {}^d\xi_i&=\text{Rot}_z\left(\psi_L\right)\Delta_i+\xi_L\\
        {}^d\psi_{i}&=\psi_L,
    \end{aligned}
    \label{eqn:ref}
\end{equation}
where {$\text{Rot}_z(\cdot)\in\mathbb{R}^{3\times3}$ is the rotation matrix around z-axis}. Equation \eqref{eqn:ref} allows the followers to compute their trajectory based on the trajectory of the virtual leader and the desired formation topology.

\section{Controller design for UAV formation} \label{controller}
\begin{figure}
    \centering    \includegraphics[width=0.8\textwidth]{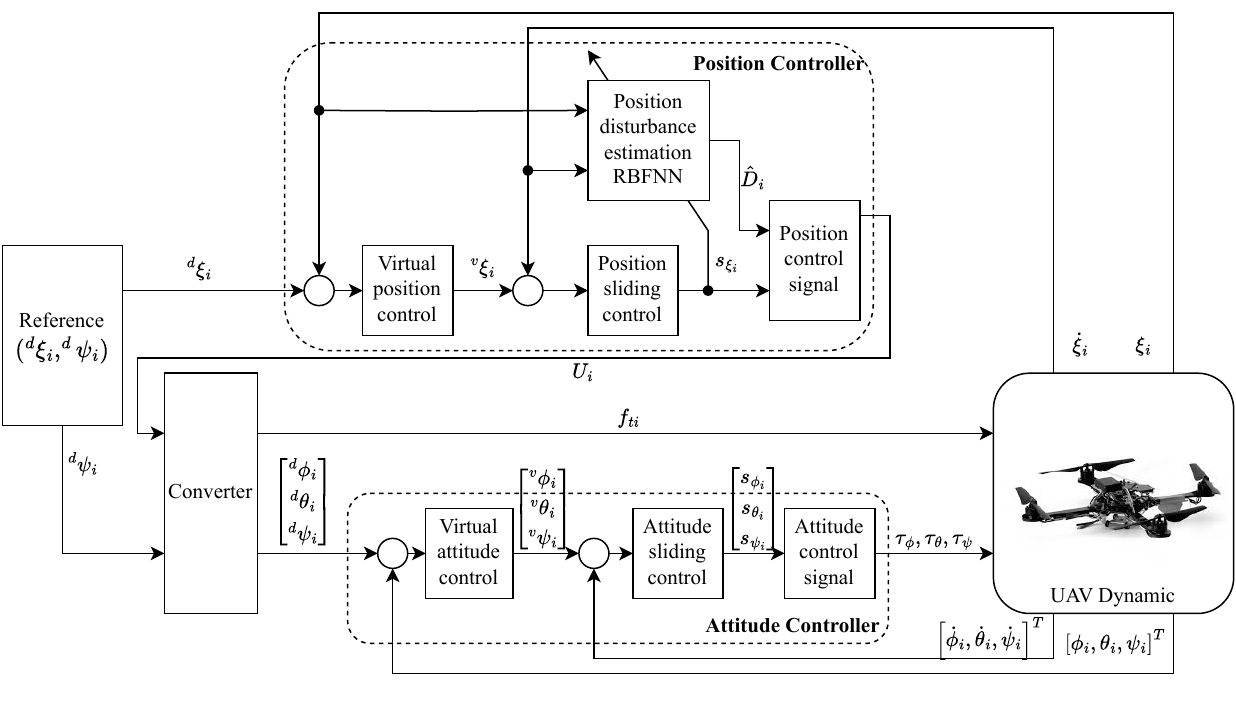}
    \caption{The proposed controller}
    \label{fig:sys}
\end{figure}

Given the trajectory of the virtual leader, denoted as $(\xi_L,\psi_L)$, the desired trajectory of follower $i$, $({}^d\xi_{i}, {}^d\psi_{i})$, in the formation can be computed based on \eqref{eqn:ref}. To track this trajectory, we design a dual-loop control system for each follower as shown in Figure \ref{fig:sys}. The outer loop is a position controller that regulates the altitude and horizontal position, while the inner loop is a backstepping sliding mode controller (BSMC) that handles the UAV's attitude, including its roll, pitch, and yaw angles. To account for external disturbances, the position controller was designed with a radial basis function neural network (RBFNN). A converter block is also included to convert the desired translational control forces into roll and pitch angles. Details of each controller are described as follows.

\subsection{Position controller design} \label{sec:outer}
The position controller aims to keep the UAV's position aligned with the desired trajectory. It is designed based on BSMC with the use of RBFNN for disturbance estimation. According to \eqref{eqn:dynamic}, dynamic equations for the translational motion of UAV $i$ are given as follows:
\begin{equation}
    \begin{aligned}
       \ddot{x}_i&=u_{xi}+\dfrac{d_{xi}}{m},\\
        \ddot{y}_i&=u_{yi}+\dfrac{d_{yi}}{m},\\
        \ddot{z}_i&=u_{zi}+\dfrac{d_{zi}}{m},
    \end{aligned}
    \label{eqn:transeqfull}
\end{equation}
where
\begin{equation}
    \begin{aligned}
        u_{xi}&=\left(\cos\phi_i\sin\theta_i\cos\psi_i+\sin\phi_i\sin\psi_i\right)\dfrac{f_{ti}}{m},\\
        u_{yi}&=\left(\cos\phi_i\sin\theta_i\sin\psi_i-\sin\phi_i\cos\psi_i\right)\dfrac{f_{ti}}{m},\\
        u_{zi}&=\cos\phi_i\cos\theta_i\dfrac{f_{ti}}{m}-g.
    \end{aligned}
    \label{eqn:trans_force}
\end{equation} 

Let $\xi_i=\left[x_i,y_i,z_i\right]^T$ and $\dot{\xi}_i=\left[\dot{x}_i,\dot{y}_i,\dot{z}_i\right]^T$ respectively be the position and velocity of the translational motion, $U_i=\left[u_{xi},u_{yi},u_{zi}\right]^T$ be the control signal, and $D_i=\left[\dfrac{d_{xi}}{m},\dfrac{d_{yi}}{m},\dfrac{d_{zi}}{m}\right]^T$,  with $\left\Vert D_i\right\Vert\leq \bar{d}$, be the external disturbance affecting UAV $i$. Equation \eqref{eqn:transeqfull} can be rewritten as:
\begin{equation}
        \ddot{\xi}_i=U_i+D_i.
    \label{eqn:transeq}
\end{equation}
The BSMC is then designed as follows.

\subsubsection{Backstepping sliding mode controller (BSMC) design}
Let ${}^e\xi_{i}$ be the translational error, ${}^e\xi_{i}=\xi_{i}-{}^d\xi_i$. The virtual velocity ,${}^v\xi_i$, of the subsystem is designed as:
\begin{equation}
    {}^v\xi_i={}^{d}\dot{\xi}_i-\lambda_{\xi} {}^{e}\xi_i,
    \label{eqn:virtual}
\end{equation}
where $\lambda_{\xi}>0$ is a positive definite gain. The first candidate Lyapunov function is chosen as
\begin{equation}
    {}^1V_{\xi_i} = \dfrac{1}{2}{}^e\xi_i^T{}^e\xi_i.
\end{equation}
Its derivative is given by
\begin{equation}
    {}^1\dot{V}_{\xi_i}={}^e\xi_{i}^T{}^e\dot{\xi}_{i}={}^e\xi_{i}^T\left(\dot{\xi}_{i}-{}^d\dot{\xi}_{i}\right).
    \label{eqn:Vp1}
\end{equation}
Substituting $\dot{\xi}_{i} = {}^v{\xi}_i$ into \eqref{eqn:Vp1} gives
\begin{equation}
    {}^1\dot{V}_{\xi_i}=-\lambda_{\xi} \left\Vert {}^e\xi_{i}\right\Vert^2\leq0.
\end{equation}
Hence, the system is stable with the virtual velocity chosen in \eqref{eqn:virtual}. The sliding mode control (SMC) algorithm is then utilized to design the input control signal for the position system. The sliding surface is chosen as follows:
\begin{equation}
    s_{\xi_i}=\gamma_{\xi} {}^e\xi_{i} + \left(\dot{\xi}_{i}-{}^v{\xi}_{i}\right),
\end{equation}
where $\gamma_{\xi}>0$ is a positive definite gain. Denote $\hat{D}_i$ as the disturbance estimated via an estimator such as the RBFNN in Section {\ref{sec:rbf}}. The derivative of $s_{\xi_i}$ then can be obtained by using $\hat{D}_i$ instead of $D_i$ as follows:
\begin{equation}
    \dot{s}_{\xi_i}=\gamma_{\xi}\left(\dot{\xi}_i-{}^d\dot{\xi}_{i}\right)+U_i+\hat{D}_i-{}^v\dot{\xi}_{i}.
    \label{eqn:sdotphi}
\end{equation}

The second Lyapunov function of the subsystem is chosen as follows:
\begin{equation}
    {}^2V_{\xi_i}=\dfrac{1}{2}s_{\xi_i}^Ts_{\xi_i}
\end{equation}

The control signals are designed as follows:

\begin{equation}
    \begin{aligned}
        U_{ieq}&={}^v\dot{\xi}_{i}-\gamma_{\xi}\left(\dot{\xi}_{i}-{}^d\dot{\xi}_{i}\right)-\hat{D}_i\\
        U_{isw}&=-\left(c_{\xi1}\text{sg}\left(s_{\xi_i}\right)+c_{\xi2}s_{\xi_i}\right),
    \end{aligned}
    \label{eqn:up}
\end{equation}
where $c_{\xi1}$ and $c_{\xi2}$ are positive gains, $U_{ieq}$ is the equivalent control signal that maintains the position variables on the sliding manifold, $U_{isw}$ is the signal that leads the subsystem to the sliding surface $s_{\xi_i}$, and $\text{sg}\left(\cdot\right)$ is the piece-wise continuous function defined as
\begin{equation}
    \text{sg}\left(x\right)=\left\{ \begin{array}{c}
    1\\
    -1\\
    \dfrac{x}{\epsilon}
    \end{array}\right.\begin{array}{c}
    x>\epsilon\\
    x<-\epsilon\\
    \text{otherwise}
    \end{array}
\end{equation}
where $0<\epsilon<1$ is a pre-defined constant.

\begin{theorem}
Consider the position control system of the UAV. If the control signal is chosen as
\begin{equation}
    U_i=U_{ieq}+U_{isw},
    \label{eqn:up0}
\end{equation}
 % \textcolor{red}{
 the system is stable. % asymptotically
\end{theorem}
\begin{proof}
Taking the first derivative of ${}^2V_{\xi_i}$ gives
\begin{equation}
    {}^2\dot{V}_{\xi_i}=s_{\xi_i}^T\dot{s}_{\xi_i}.
    \label{eqn:Vdotp20}
\end{equation}
Substituting \eqref{eqn:sdotphi} into \eqref{eqn:Vdotp20} gives
\begin{equation}
    {}^2\dot{V}_{\xi_i}=s_{\xi_i}^T\left(\gamma_{\xi}\left(\dot{\xi}_{i}-{}^d\dot{\xi}_{i}\right)+U_i+\hat{D}_i-{}^v\dot{\xi}_{i}\right).
    \label{eqn:Vdotp2}
\end{equation}
By substituting \eqref{eqn:up} and \eqref{eqn:up0} into \eqref{eqn:Vdotp2}, $\dot{V}_{\xi_i}$ becomes

\begin{equation}
    {}^2\dot{V}_{\xi_i}=-c_{\xi1}s_{\xi_i}^T\text{sg}\left(s_{\xi_i}\right)-c_{\xi2}\left\Vert s_{\xi_i}\right\Vert^2\leq0.    
\end{equation}
According to Lyapunov's stability theorem, the system is stable.
\end{proof}
\subsubsection{Radial basis function neural network (RBFNN) design} \label{sec:rbf}
\begin{figure}
    \centering
    \includegraphics[width=0.45\textwidth]{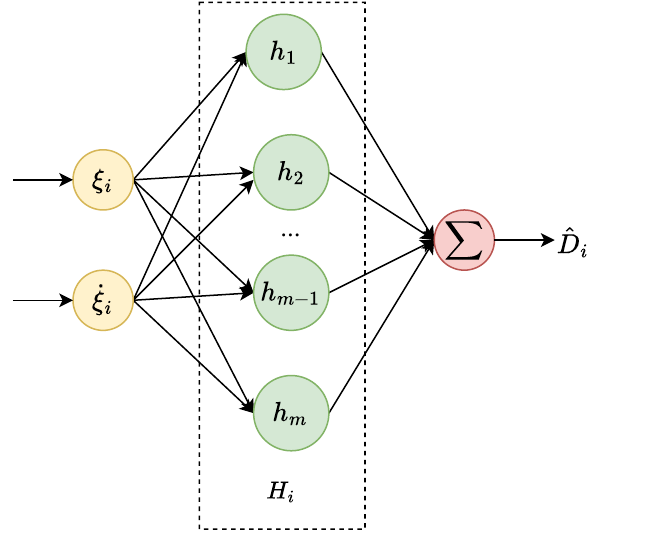}
    \caption{The RBFNN structure}
    \label{fig:rbfnn}
\end{figure}

During operation, UAVs are subject to inevitable disturbances such as wind or turbulent flows. Those disturbances affect the system performance, but are complex to model and analyze. We address this problem by exploiting the online learning capability of neural networks to estimate the disturbances. Previous studies on universal approximation theorems for RBFNN show that RBFNN can approximate any nonlinear function on a bounded set with an arbitrary level of accuracy \cite{6796390}. In this work, we design a disturbance estimator using a neural network with the radial basis function (RBF). The network has three layers including an input layer, a hidden layer, and an output layer, as shown in Figure \ref{fig:rbfnn}. Position vector $\xi_i$ and its derivation $\dot{\xi}_i$ are the input of the network. At the hidden layer, neurons are activated by a radial basis function. The output of neuron $j$ is computed as:
\begin{equation}
    h_{j}=\exp{\left(-\dfrac{\left\Vert \xi_{i}-\mu_{1j}\right\Vert^2+\left\Vert \dot{\xi}_{i}-\mu_{2j}\right\Vert^2}{b^2}\right)},
    \label{eqn:h}
\end{equation}
where $b$ is a parameter controlling the width of the Gaussian function, $\mu_{1j}$ and $\mu_{2j}$ are predefined center points, and $j\in\{1,2,...,m\}$ is the neuron index with $m$ being the number of neurons in the hidden layer. The output layer is a weighted sum. Let $W_i$ be the optimal weight matrix, $H_i$ be the output of the hidden layer, and  $\sigma_i$ be the approximation error. Disturbance $D_i$ affecting UAV $i$ then can be expressed by:
\begin{equation} 
    D_i=W_i^TH_i+\sigma_i
    \label{eqn:D}
\end{equation}

The output $\hat{D}_i$ of the RBFNN approximates $D_i$ as:
\begin{equation}
    \hat{D}_i=\hat{W}_i^TH_i,
    \label{eqn:hatD}
\end{equation}
where $\hat{W}_i$ is a trained weight matrix. This matrix is updated based on the following rule:
% \textcolor{red}{(Kiem tra lai thuat ngu ideal weight matrix and updated weight matrix)}. 
\begin{equation}
    \dot{\hat{W}}_i=a\left(H_is_{\xi_i}^T-\eta\left\Vert s_{\xi_i}\right\Vert\hat{W}_i\right),
    \label{eqn:dotw}
\end{equation}
where $a$ is a positive definite gain matrix. With this structure, the estimation of disturbance $D_i$ can be described as in Algorithm \ref{alg:rbf}, where $\beta\in(0,1)$ is the momentum factor.

\begin{algorithm}
\caption{Pseudocode to estimate disturbances by RBFNN for UAV $i$}\label{alg:rbf}
\tcc{Initialization only one time}
Initialize parameters $m$, $a$, $b$, $\eta$, $\beta$\;
Initialize predefined center points $\mu_{1}$, $\mu_{2}$\;
Create random weight matrix $\hat{W}_i$\;
Initialize stored weight matrices $\hat{W}_{1i}=\hat{W}_{2i}=\hat{W}_{i}$\;

\tcc{Online training}
\ForEach{computation step}
{
    Get the current values of $\xi_i$, $\dot{\xi}_i$, $s_{\xi_i}$\;
    $H_i = \text{zeros}(m,1)$\;
    \For{$j=1$ to $m$}
    {
        Compute sub-hidden layer $h_j$\tcc*[r]{Equation  \ref{eqn:h}}
    }
    Compute the value of $\dot{\hat{W}}_i$\tcc*[r]{Equation  \ref{eqn:dotw}}
    Update the weight matrix $\hat{W}_i = \hat{W}_{1i} + \dot{\hat{W}}_i + \beta(\hat{W}_{1i}-\hat{W}_{2i})$\;
    Update stored weight matrices $\hat{W}_{2i}=\hat{W}_{1i}$ and $\hat{W}_{1i}=\hat{W}_{i}$\;
    Compute the current estimated disturbance $\hat{D}_i$\tcc*[r]{Equation  \ref{eqn:hatD}}
}
\end{algorithm}

\subsubsection{Stability of the position controller} \label{sec:sta}
% By using the estimated disturbance $\hat{D}_i$, the control signal of the position controller in \eqref{eqn:up} can be rewritten as:
% \begin{equation}
%     \begin{aligned}
%         U_{ieq}&={}^v\dot{\xi}_{i}-\gamma_{\xi}\left(\dot{\xi}_{i}-{}^d\dot{\xi}_{i}\right)-\hat{D}_i\\
%         U_{isw}&=-\left(c_{\xi1}\text{sg}\left(s_{\xi_i}\right)+c_{\xi2}s_{\xi_i}\right)
%     \end{aligned}
%     \label{eqn:upnew}
% \end{equation}
The stability of this controller is addressed in Theorem \ref{theorem2} as follows.

\begin{theorem}
Consider UAV $i$ affected by external disturbance $D_i$ as described in \eqref{eqn:transeq}, the control signals designed in \eqref{eqn:up}, the bounded weight  $\left\Vert W_i\right\Vert\leq\bar{W}_i$, and the update rule for  RBFNN in \eqref{eqn:dotw}. If the following inequality condition is satisfied:
\begin{equation}
    \left\Vert s_{\xi_i}\right\Vert \geq\dfrac{\bar{\sigma}+\dfrac{1}{4}\eta\bar{W}_i^{2}}{c_{\xi2}},
    \label{eqn:ineq}
\end{equation}
the position control system is stable. 

\label{theorem2}
\end{theorem}
\begin{proof}
Choose the candidate Lyapunov function as follows:
\begin{equation}
    V_{\xi_i}={}^2V_{\xi_i}+\dfrac{1}{2}\text{tr}\left(\tilde{W}_i^Ta^{-1}\tilde{W}_i\right),
\end{equation}
where $\tilde{W}_i=W_i-\hat{W}_i$ is the error weight matrix. Taking the first derivative of $V_{\xi_i}$ gives:
\begin{equation}
    \begin{aligned}
        \dot{V}_{\xi_i}&=s_{\xi_i}^{T}\dot{s}_{\xi_i}+\text{tr}\left(\tilde{W}_i^{T}a^{-1}\dot{\tilde{W}_i}\right)\\
        &=-c_{\xi1}s_{\xi_i}^T\text{sg}\left(s_{\xi_i}\right) -c_{\xi2}\left\Vert s_{\xi_i}\right\Vert ^{2}-s_{\xi_i}^{T}\left(\hat{D}_i-D_{i}\right)-\text{tr}\left(\tilde{W}_i^{T}a^{-1}\dot{\hat{W}}_i\right)\\
        &=-c_{\xi1}s_{\xi_i}^T\text{sg}\left(s_{\xi_i}\right) -c_{\xi2}\left\Vert s_{\xi_i}\right\Vert ^{2}+s_{\xi_i}^{T}\sigma_i+s_{\xi_i}^{T}\tilde{W}_i^{T}H_i-\text{tr}\left(\tilde{W}_i^{T}a^{-1}\dot{\hat{W}}_i\right)\\
        &=-c_{\xi1}s_{\xi_i}^T\text{sg}\left(s_{\xi_i}\right) -c_{\xi2}\left\Vert s_{\xi_i}\right\Vert ^{2}+s_{\xi_i}^{T}\sigma_i+\text{tr}\left(-\tilde{W}_i^{T}\left(a^{-1}\dot{\hat{W}}_i-H_is_{\xi_i}^{T}\right)\right)
    \end{aligned}
\end{equation}
With the updated rule of the neural network, $\dot{V}_{\xi_i}$ can be rewritten as follows:
\begin{equation}
    \begin{aligned}
        \dot{V}_{\xi_i}&=-c_{\xi1}s_{\xi_i}^T\text{sg}\left(s_{\xi_i}\right)-c_{\xi2}\left\Vert s_{\xi_i}\right\Vert ^{2}+s_{\xi_i}^{T}\sigma_i+\text{tr}\left(\tilde{W}_i^{T}\eta\left\Vert s_{\xi_i}\right\Vert \hat{w}_i\right)\\
        &=-c_{\xi1}s_{\xi_i}^T\text{sg}\left(s_{\xi_i}\right) -c_{\xi2}\left\Vert s_{\xi_i}\right\Vert ^{2}+s_{\xi_i}^{T}\sigma_i+\eta\left\Vert s_{\xi_i}\right\Vert \text{tr}\left(\tilde{W}_i^{T}\left(W_i-\tilde{W}_i\right)\right)
    \end{aligned}
\end{equation}
According to the Cauchy-Schwarz inequality, the following inequality equation can be satisfied:
\begin{equation}
    \text{tr}\left(\tilde{W}_i^{T}\left(W_i-\tilde{W}_i\right)\right)\leq\left\Vert \tilde{W}_i\right\Vert \left\Vert W_i\right\Vert -\left\Vert \tilde{W}_i\right\Vert ^{2}
\end{equation}
Thus,
\begin{equation}
    \begin{aligned}
        \dot{V}_{\xi_i}&\leq-c_{\xi1}s_{\xi_i}^T\text{sg}\left(s_{\xi_i}\right)-c_{\xi2}\left\Vert s_{\xi_i}\right\Vert ^{2}+\left\Vert s_{\xi_i}\right\Vert \bar{\sigma}+\eta\left\Vert s_{\xi_i}\right\Vert \left(\left\Vert \tilde{W}_i\right\Vert \left\Vert W_i\right\Vert -\left\Vert \tilde{W}_i\right\Vert ^{2}\right)\\
        &\leq-c_{\xi1}s_{\xi_i}^T\text{sg}\left(s_{\xi_i}\right)-c_{\xi2}\left\Vert s_{\xi_i}\right\Vert ^{2}+\left\Vert s_{\xi_i}\right\Vert \bar{\sigma}+\dfrac{1}{4}\eta\left\Vert s_{\xi_i}\right\Vert \bar{W}_i^{2}-\eta\left\Vert s_{\xi_i}\right\Vert \left(\dfrac{1}{2}\left\Vert W_i\right\Vert -\left\Vert \tilde{W}_i\right\Vert \right)^{2}
    \end{aligned}
\end{equation}
Based on an extension of the Lyapunov theorem \cite{NARENDRA1987}, $\left\Vert s_{\xi_i}\right\Vert$ is bounded. Moreover, the control gain $c_{\xi2}$ can be selected large enough so that
\begin{equation}
    \left[{\bar{\sigma}+\eta\bar{W}_i^{2}/{4}}\right]/{c_{\xi2}}\leq b_\xi
\end{equation}
Therefore, with the inequality condition \eqref{eqn:ineq}, $\dot{V}_{p}$ can be rewritten as follows:
\begin{equation}
    \dot{V}_{\xi_i}\leq-c_{\xi1}s_{\xi_i}^T\text{sg}\left(s_{\xi_i}\right)-\eta\left\Vert s_{\xi_i}\right\Vert \left(\dfrac{1}{2}\left\Vert W_i\right\Vert -\left\Vert \tilde{W}_i\right\Vert \right)^{2}\leq 0
\end{equation}
The Lyapunov stability condition is satisfied.
\end{proof}

\subsection{Attitude controller design} \label{sec:inner}
In our system, the position controller is the outer loop of the UAV control system, as depicted in Figure \ref{fig:sys}. Its control signal is then fed to the converter block to calculate the desired angles and translational forces based on \eqref{eqn:trans_force} as:
\begin{equation}
    \begin{aligned}
        {}^{d}\theta_i&=\arctan\left(\dfrac{u_{xi}\cos{}^{d}\psi_i+u_{yi}\sin{}^d\psi_{i}}{u_{zi}+g}\right)\\
        {}^d\phi_{i}&=\arctan\left(\cos{}^d\theta_{i}\dfrac{u_{xi}\sin{}^d\psi_{i}-u_{yi}\cos{}^d\psi_{i}}{u_{zi}+g}\right)\\
        f_{ti}&=\dfrac{u_{zi}+g}{\cos{}^d\phi_{i}\cos{}^d\theta_{i}}
    \end{aligned}
    \label{eqn:converter}
\end{equation}
They are used as the reference for the attitude controller, which is designed based on the BSMC. From \eqref{eqn:dynamic}, the dynamic equation for the roll angle is given by:
\begin{equation}
    \begin{aligned}
        \ddot{\phi}_{i}&=\dfrac{\dot{\theta}_{i}\dot{\psi}_{i}\left(I_{y}-I_{z}\right)+\tau_{\phi_i}}{I_{x}}.
    \end{aligned}
    \label{eqn:sysroll}
\end{equation}
Denote ${}^e\phi_{i} = \phi_i -{}^d\phi_{i}$ as the roll angle error. The virtual velocity, ${}^v\phi_{i}$, is defined as:
\begin{equation}
    {}^v\phi_{i} = {}^d\dot{\phi}_{i} - \lambda_\phi \phi_{1e},
\end{equation}
where $\lambda_\phi>0$ is a positive gain. The first candidate Lyapunov function for subsystem $\phi_{1e}$ is chosen as:
\begin{equation}
    {}^1V_{\phi_i} = \dfrac{1}{2}{}^e\phi_{i}^2.
\end{equation}
Taking the first derivative of ${}^1V_{\phi}$ gives:
\begin{equation}
    {}^1\dot{V}_{\phi_i}={}^e\phi_{i}{}^e\dot{\phi}_{i}={}^e\phi_{i}\left(\dot{\phi}_i-{}^d\dot{\phi}_{i}\right).
    \label{eqn:Vphi1}
\end{equation}
Substituting $\dot{\phi}_i = {}^v\phi_{i}$ into \eqref{eqn:Vphi1} gives 
\begin{equation}
{}^1\dot{V}_{\phi_i}=-\lambda_\phi {}^e\phi_{i}^2\leq0.
\end{equation}
Thus, the Lyapunov stability is guaranteed. The sliding surface of the roll angle subsystem is expressed as:
\begin{equation}
    s_{\phi_i}=\gamma_{\phi} {}^e\phi_{i} + \left(\dot{\phi}_i - {}^e\phi_i\right),
\end{equation}
where $\gamma_\phi>0$ is a positive gain. The first derivative of $s_{\phi_i}$ is given by:
\begin{equation}
    \dot{s}_{\phi_i}=\gamma_\phi\left(\dot{\phi}_i-{}^d\dot{\phi}_{i}\right)+\dfrac{\dot{\theta}_{i}\dot{\psi}_{i}\left(I_{y}-I_{z}\right)+\tau_{\phi_i}}{I_{x}}-{}^v\dot{\phi}_{i}.
    \label{eqn:dotsphi}
\end{equation}
The control signal is then designed with two sub-control signals, ${}^{eq}\tau_{\phi_i}$ and ${}^{sw}\tau_{\phi_i}$. ${}^{eq}\tau_{\phi_i}$ is the equivalent control signal that maintains the roll angle on the sliding manifold and ${}^{sw}\tau_{\phi_i}$ is the signal that leads the subsystem to the sliding surface $s_{\phi_i}$. They are chosen as follows:
\begin{equation}
    \begin{aligned}
        {}^{eq}\tau_{\phi_i}&=I_x\left({}^v\dot{\phi}_{i}-\gamma_\phi\left(\dot{\phi}_i-{}^d\dot{\phi}_{i}\right)\right)-\dot{\theta}_{i}\dot{\psi}_{i}\left(I_{y}-I_{z}\right)\\
        {}^{sw}\tau_{\phi_i}&=-I_x\left(c_{\phi 1}\text{sg}\left(s_{\phi_i}\right)+c_{\phi 2}s_{\phi_i}\right),
    \end{aligned}
    \label{eqn:uphi}
\end{equation}
where $c_{\phi 1}$ and $c_{\phi 2}$ are positive gains.

\begin{theorem}
Consider the roll angle subsystem \eqref{eqn:sysroll}. If the control signal is designed as:
\begin{equation}
    \tau_{\phi_i}={}^{eq}\tau_{\phi_i}+{}^{sw}\tau_{\phi_i},
    \label{eqn:uphi0}
\end{equation}
the roll angle control system is stable.
\end{theorem}
\begin{proof}
The candidate Lyapunov function of the roll angle subsystem is chosen as follows:
\begin{equation}
    V_{\phi_i}=\dfrac{1}{2}s_{\phi_i}^2.
\end{equation}
Taking the first derivative of $V_\phi$ gives
\begin{equation}
    \dot{V}_{\phi_i}=s_{\phi_i}\dot{s}_{\phi_i}.
    \label{eqn:Vdotphi0}
\end{equation}
By substituting \eqref{eqn:dotsphi} into \eqref{eqn:Vdotphi0}, we have
\begin{equation}
    \dot{V}_{\phi_i}=s_{\phi_i}\left(\gamma_\phi\left(\dot{\phi}_i-{}^d\dot{\phi}_{i}\right)+\dfrac{\dot{\theta}_{i}\dot{\psi}_{i}\left(I_{y}-I_{z}\right)+\tau_{\phi_i}}{I_{x}}-{}^v\dot{\phi}_{i}\right).
    \label{eqn:Vdotphi}
\end{equation}
Finally, substituting \eqref{eqn:uphi} and \eqref{eqn:uphi0} into \eqref{eqn:Vdotphi} gives
\begin{equation}
    \dot{V}_{\phi_i}=-c_{\phi 1} s_{\phi_i}\text{sg}\left(s_{\phi_i}\right)-c_{\phi 2}s_{\phi_i}^2\leq0.
    \label{eqn:finalVdot}
\end{equation}
Thus, the Lyapunov stability of the roll angle control system is guaranteed.
\end{proof}

The control signals for the pitch and yaw angles can be obtained by applying the design process similar to the roll angle. As a result, the pitch control signals are obtained as: 
\begin{equation}
    \begin{aligned}
        \tau_{\theta_i}&={}^{eq}\tau_{\theta_i}+{}^{sw}\tau_{\theta_i}\\
        {}^{eq}\tau_{\theta_i}&=I_y\left({}^v\dot{\theta}_{i}-\gamma_\theta\left(\dot{\theta}_i-{}^d\dot{\theta}_{i}\right)\right)-\dot{\phi}_{i}\dot{\psi}_{i}\left(I_{z}-I_{x}\right)\\
        {}^{sw}\tau_{\theta_i}&=-I_y\left(c_{\theta 1}\text{sg}\left(s_{\theta_i}\right)+c_{\theta 2}s_{\theta_i}\right),
    \end{aligned}
    \label{eqn:utheta}
\end{equation}
and the yaw control signals are given by:
\begin{equation}
    \begin{aligned}
        \tau_{\psi_i}&={}^{eq}\tau_{\psi_i}+{}^{sw}\tau_{\psi_i}\\
        {}^{eq}\tau_{\psi_i}&=I_z\left({}^v\dot{\psi}_{i}-\gamma_\psi\left(\dot{\psi}_i-{}^d\dot{\psi}_{i}\right)\right)-\dot{\phi}_{i}\dot{\theta}_{i}\left(I_{x}-I_{y}\right)\\
        {}^{sw}\tau_{\psi_i}&=-I_z\left(c_{\psi 1}\text{sg}\left(s_{\psi_i}\right)+c_{\psi 2}s_{\psi_i}\right).
    \end{aligned}
    \label{eqn:upsi}
\end{equation}

\section{Results}\label{result}
\begin{table}
\begin{center}
\caption{Parameters of the Hummingbird quadrotor}
\label{tbl:uav_param}
\begin{tabular}{C{7.5cm} C{2.5cm} C{2.5cm}}
\hline
Description & Notation & Value      \\ \hline
\multicolumn{1}{l}{The mass of UAV (kg)} & $m$             & 0.68                     \\ 
\multicolumn{1}{l}{The gravitational acceleration (m/s$^2$)}       &    $g$             & 9.81               \\ 
\multicolumn{1}{l}{Moment of inertia about x and y axes (kg.m$^2$)} &   $I_x$, $I_y$    & 0.007         \\
\multicolumn{1}{l}{Moment of inertia about z-axis (kg.m$^2$)}   &   $I_z$           & 0.012           \\
\multicolumn{1}{l}{Length of UAV arm (m)}           &   $l$             & 0.17                    \\
\multicolumn{1}{l}{Thrust coefficient}             &    $k_t$           & 29$\times$10$^{-6}$       \\
\multicolumn{1}{l}{Drag coefficient}        &    $k_d$           & 1.1$\times$10$^{-6}$          \\ \hline
\end{tabular}
\end{center}
\end{table}

\begin{table}
\caption{Parameters of the proposed controller}
\label{tbl:control_param}
\centering
\begin{tabular}{m{4cm} m{9cm}}
\hline
\multicolumn{2}{c}{Parameters of the BSMC }                                 \\ \hline
Attitude controller & \begin{tabular}[c]{@{}l@{}}$\lambda_{\phi}=\lambda_{\theta}=\lambda_{\psi}=5$;\\ $\gamma_{\phi}=\gamma_{\theta}=\gamma_{\psi}=5$;\\ $c_{\phi 1}=c_{\theta 1}=c_{\psi 1}=2$;\\ $c_{\phi 2}=c_{\theta 2}=c_{\psi 2}=2$;\end{tabular} \\ \hline
Position controller & $\lambda_{\xi}=3$; $\gamma_\xi=2$; $c_{\xi1}=c_{\xi2}=2$                                  \\ \hline
\multicolumn{2}{c}{Parameters of the RBFNN}                                \\ \hline
Position controller & \begin{tabular}[c]{@{}l@{}}$m=50;a=0.1;b=10;\eta=0.2;$ \\ $\left\{ \begin{array}{c}
\mu_{x}=\text{linspace}\left(-r_{x},r_{x},n\right)\\
\mu_{y}=\text{linspace}\left(-r_{y},r_{y},n\right)\\
\mu_{z}=\text{linspace}\left(-r_{z},r_{z},n\right)
\end{array}\right.;r_{x}=r_{y}=r_{z}=15;$\\$\mu_{1}=\left[\mu_{x},\mu_{y},\mu_{z}\right]^{T};\mu_{2}=\mu_{1}$;\end{tabular}  \\ \hline
\end{tabular}
\end{table}

\begin{figure}
\centering
\includegraphics[width=0.25\textwidth]{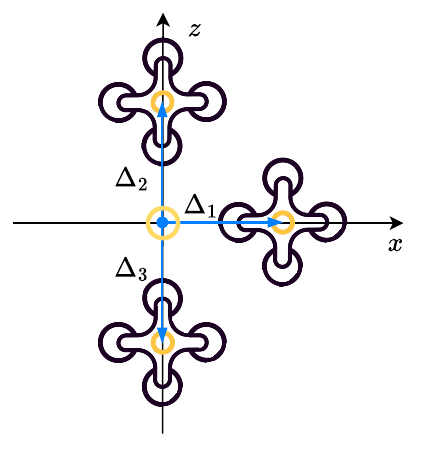}
\caption{The desired topology}
\label{fig:topology}
\end{figure}

\begin{figure}
    \centering
    \includegraphics[width=0.8\textwidth]{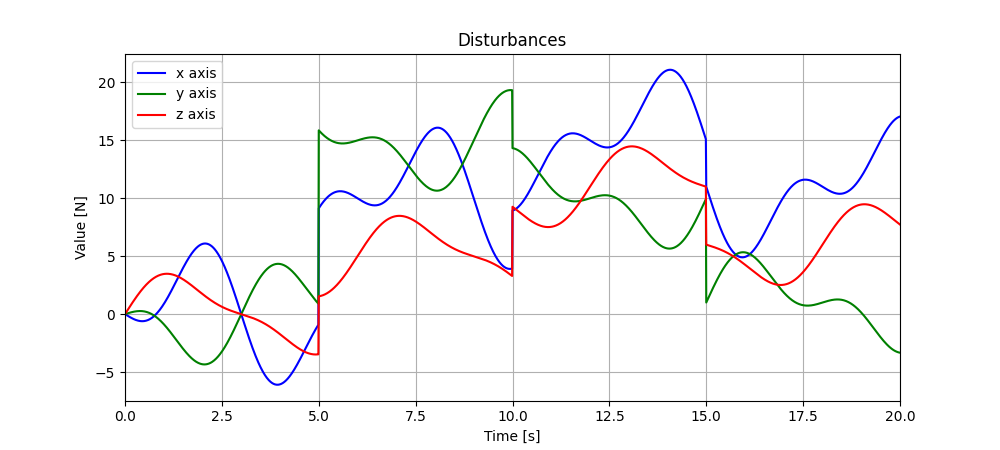}
    \caption{The external disturbance acting on the formation generated based on the rectangle and full wavelength “1-cosine” wind model in Scenario 1}
    \label{fig:ex_dis}
\end{figure}

\begin{figure}
\centering
    \begin{subfigure}[b]{0.45\textwidth}
    \centering
    \includegraphics[width=\textwidth]{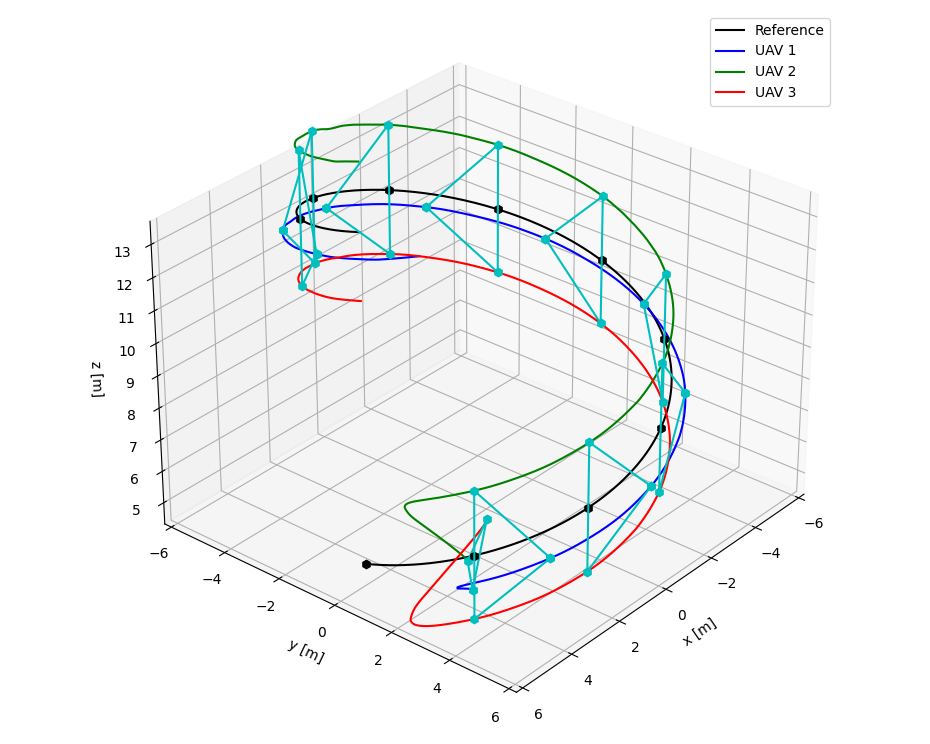}
    \caption{RBF-BSMC}
    \end{subfigure}
    \begin{subfigure}[b]{0.45\textwidth}
    \centering
    \includegraphics[width=\textwidth]{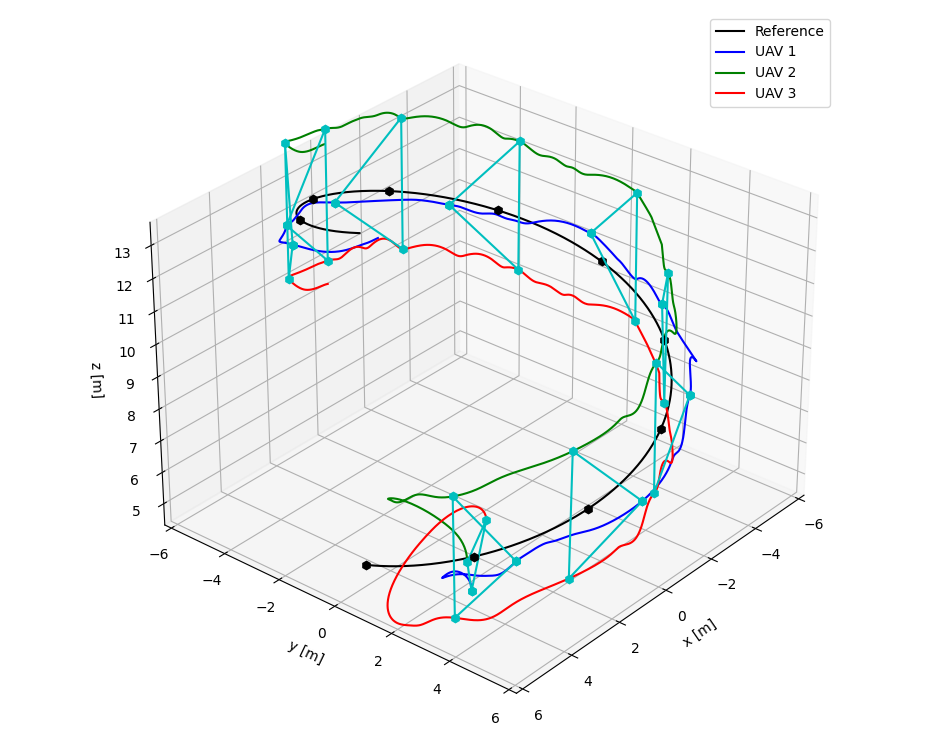}
    \caption{MPC}
    \end{subfigure}
    \begin{subfigure}[b]{0.45\textwidth}
    \centering
    \includegraphics[width=\textwidth]{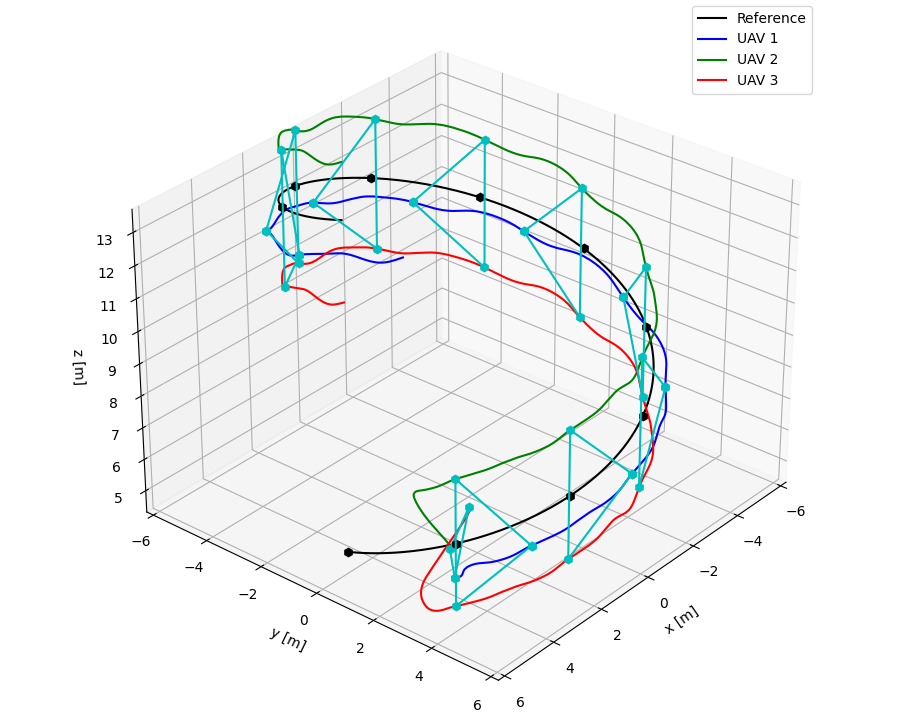}
    \caption{BSMC}
    \end{subfigure}
    \begin{subfigure}[b]{0.45\textwidth}
    \centering
    \includegraphics[width=\textwidth]{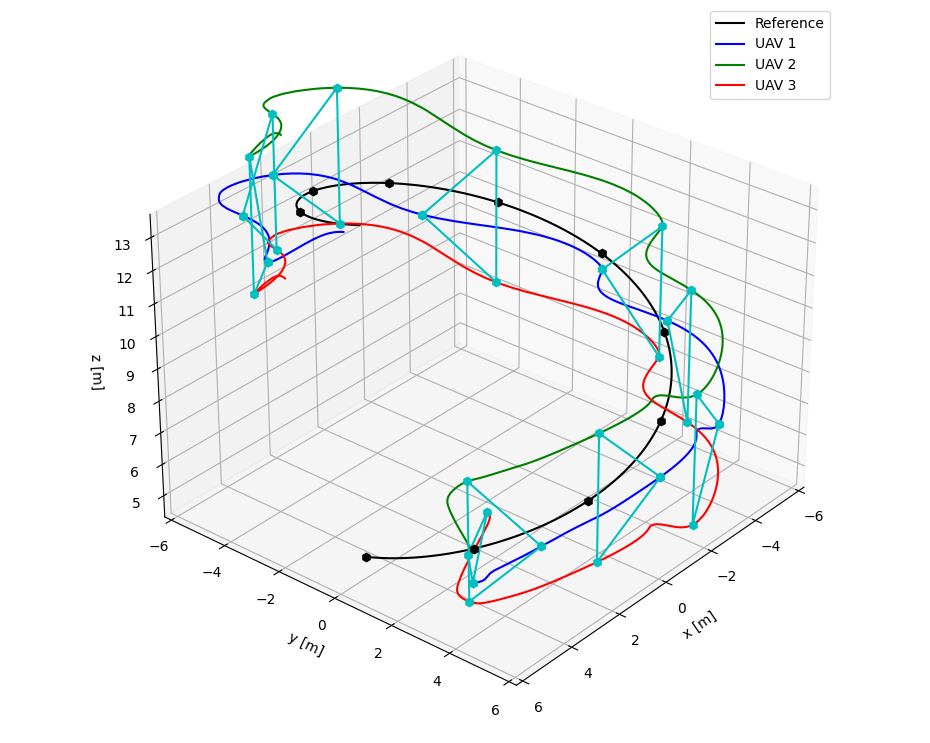}
    \caption{SMC}
    \end{subfigure}
    \caption{Trajectories of the UAV formation generated by the four controllers in Scenario 1}
    \label{fig:tracking}
\end{figure}

\begin{figure}
    \centering
    \begin{subfigure}[b]{0.49\textwidth}
    \centering
    \includegraphics[width=\textwidth]{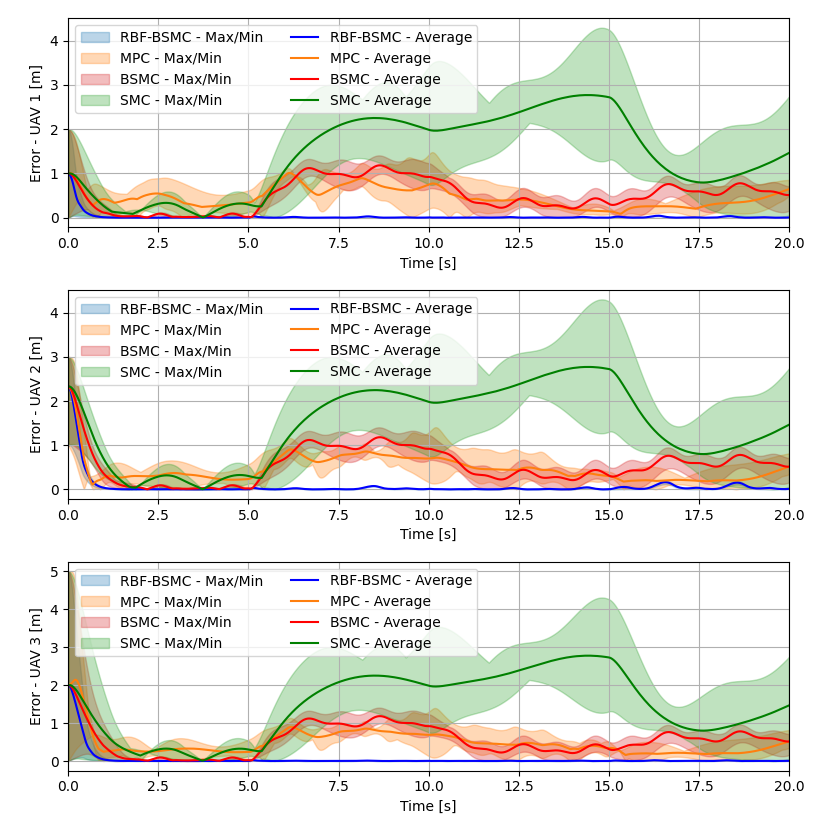}
    \caption{The tracking errors}
    \label{fig:error}
    \end{subfigure}
    \begin{subfigure}[b]{0.49\textwidth}
    \centering
    \includegraphics[width=\textwidth]{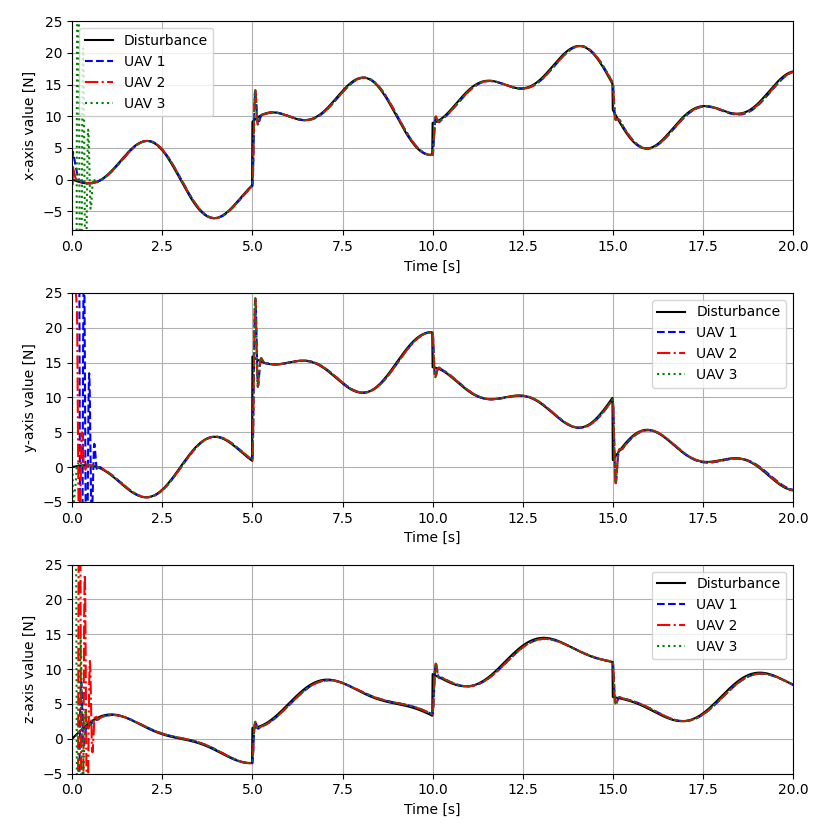}
    \caption{The estimated disturbances}
    \label{fig:est}
    \end{subfigure}
    \caption{The tracking errors and estimated disturbances of the controllers in Scenario 1}
    \label{fig:scen1}
\end{figure}

\begin{figure}
    \centering
    \includegraphics[width=0.8\textwidth]{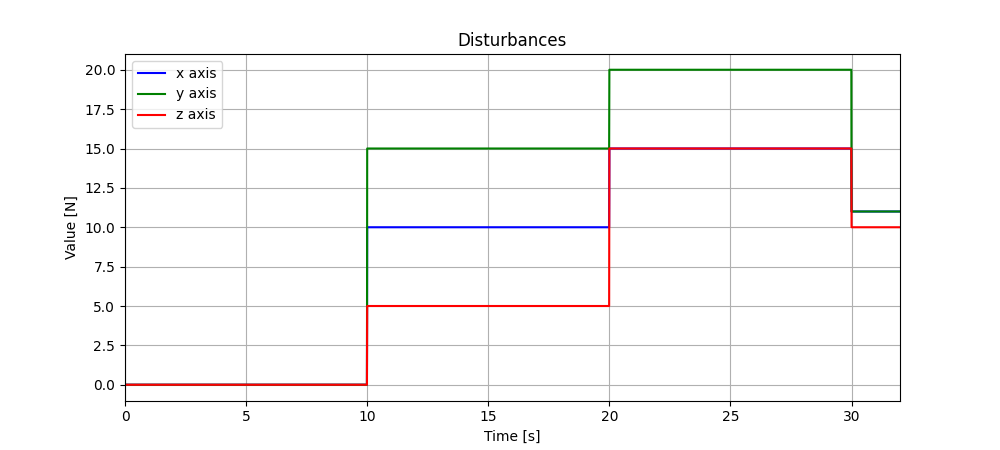}
    \caption{The external disturbance acting on the formation generated based on the rectangle wind model in Scenario 2}
    \label{fig:ex_dis2}
\end{figure}

\begin{figure}
\centering
    \begin{subfigure}[b]{0.49\textwidth}
    \centering
    \includegraphics[width=\textwidth]{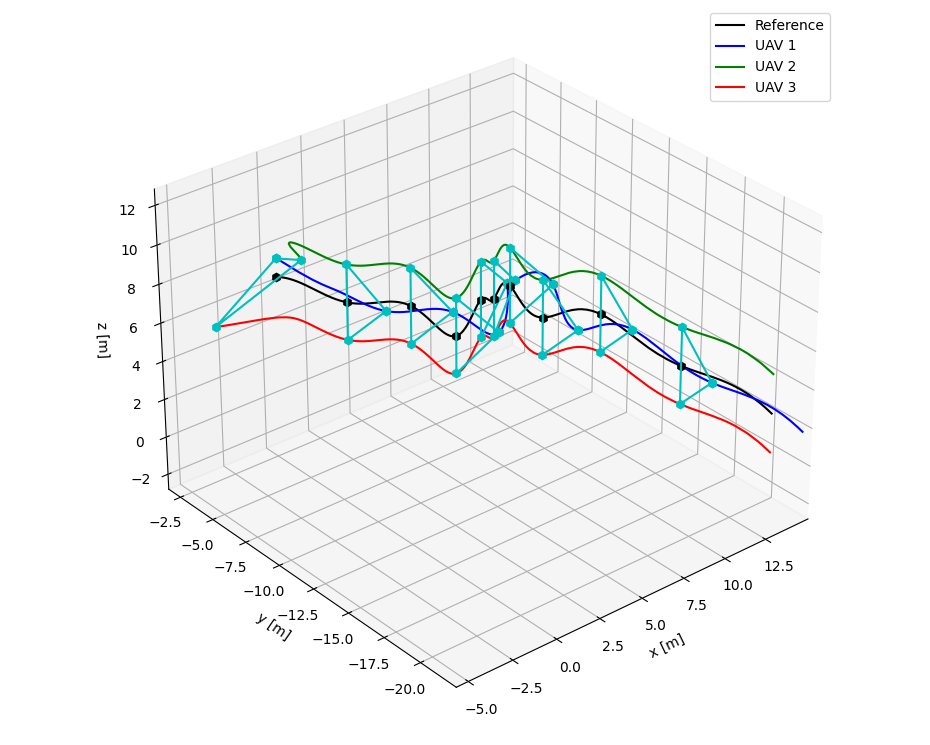}
    \caption{RBF-BSMC}
    \end{subfigure}
    \begin{subfigure}[b]{0.49\textwidth}
    \centering
    \includegraphics[width=\textwidth]{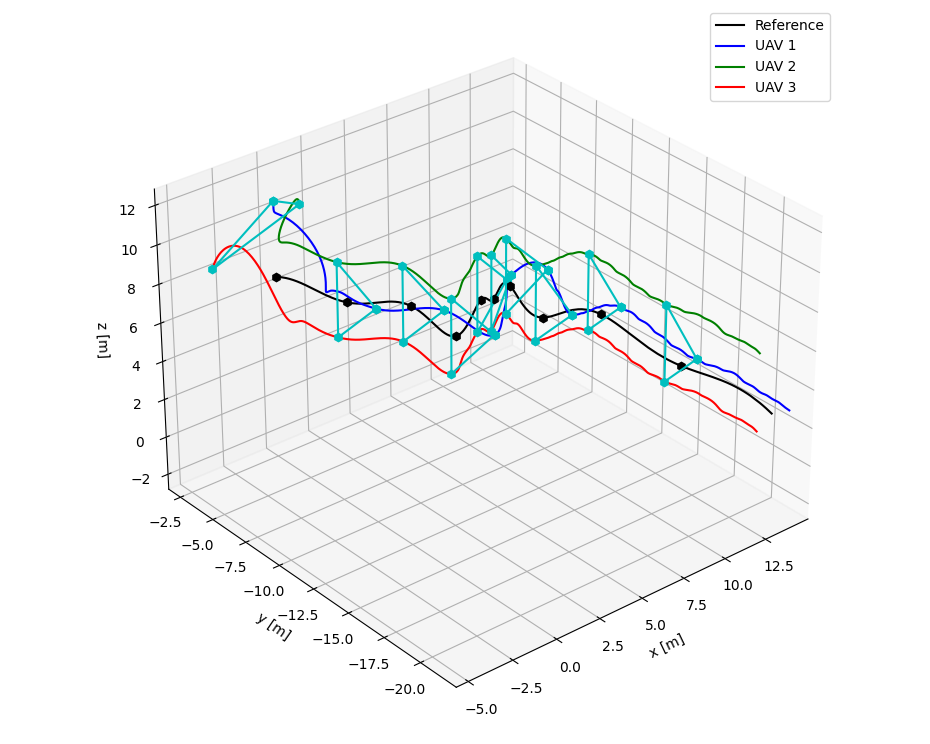}
    \caption{MPC}
    \end{subfigure}
    \begin{subfigure}[b]{0.49\textwidth}
    \centering
    \includegraphics[width=\textwidth]{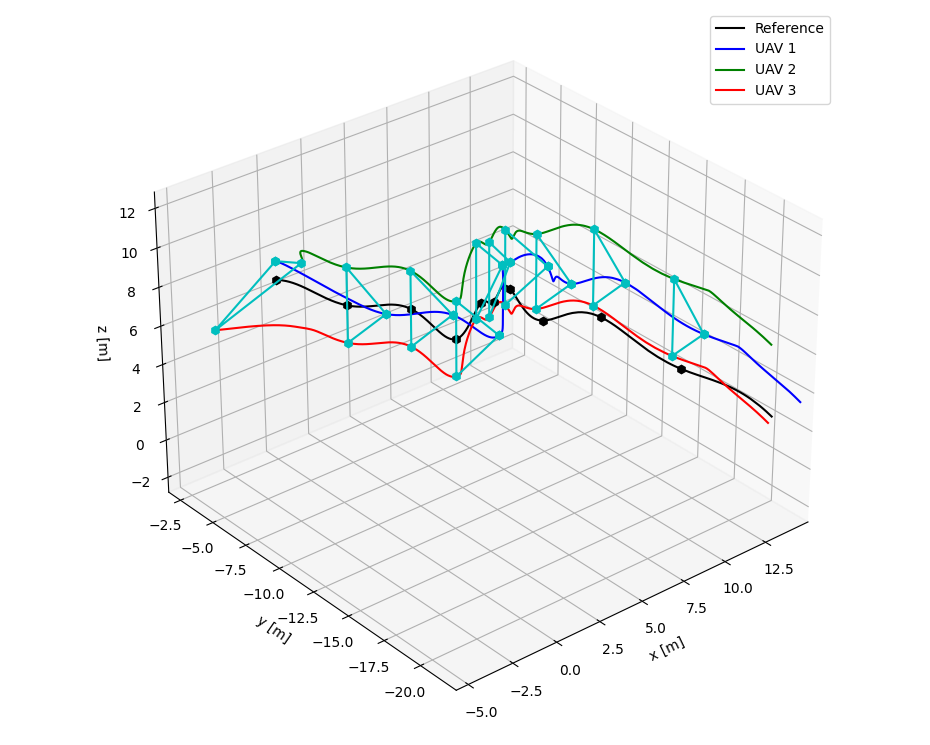}
    \caption{BSMC}
    \end{subfigure}
    \begin{subfigure}[b]{0.49\textwidth}
    \centering
    \includegraphics[width=\textwidth]{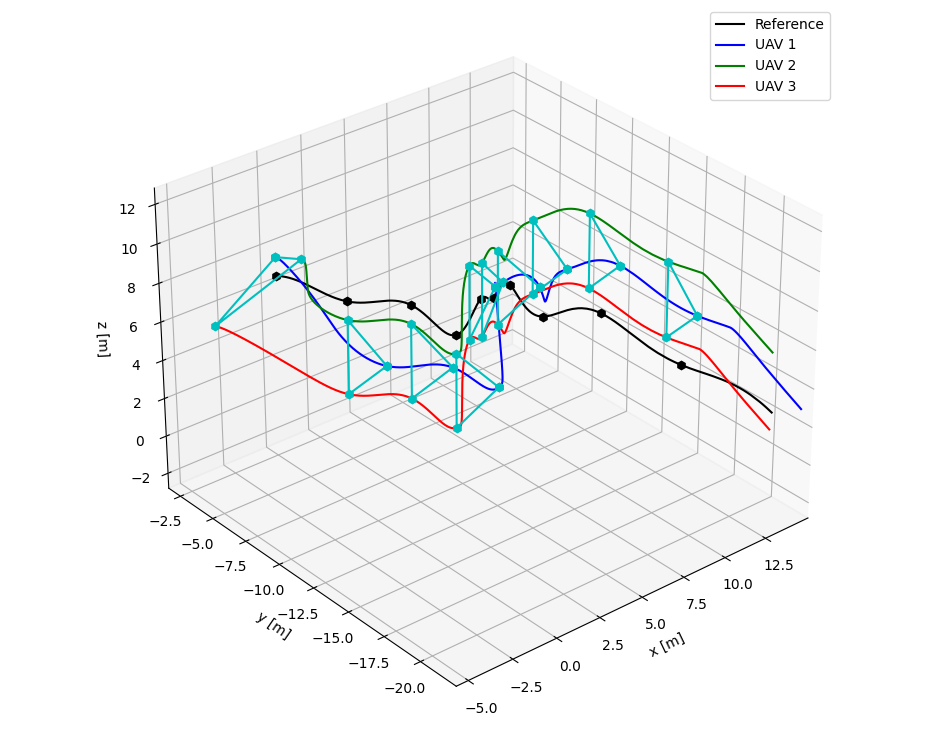}
    \caption{SMC}
    \end{subfigure}
    \caption{Trajectories of the UAV formation generated by the four controllers in Scenario 2}
    \label{fig:tracking2}
\end{figure}

\begin{figure}
    \centering
    \begin{subfigure}[b]{0.48\textwidth}
    \centering
    \includegraphics[width=\textwidth]{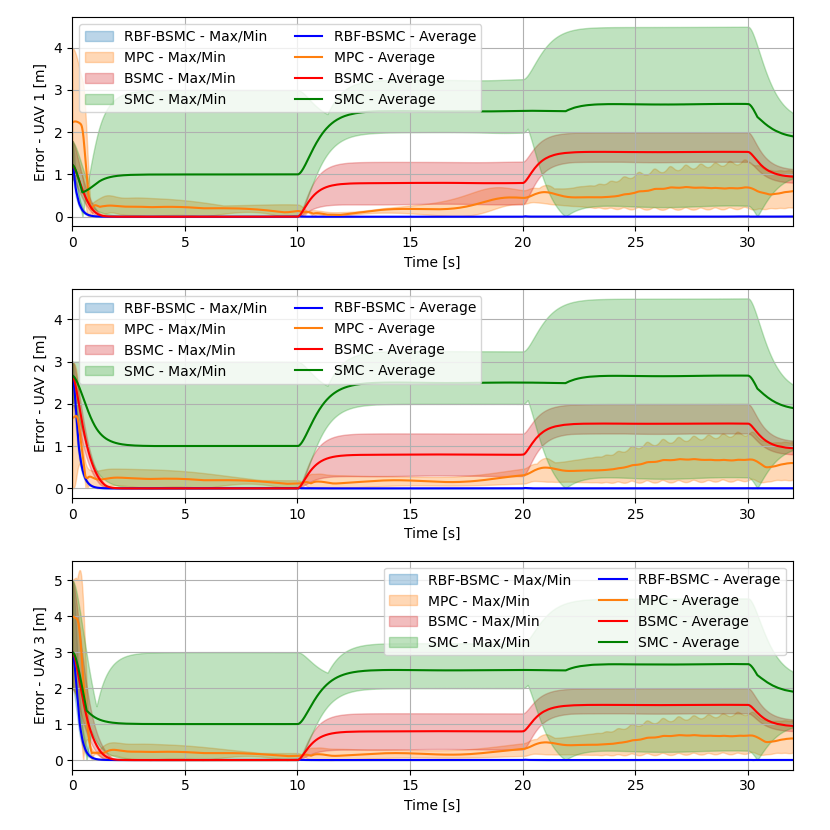}
    \caption{The tracking errors}
    \label{fig:error2}
    \end{subfigure}
    \begin{subfigure}[b]{0.48\textwidth}
    \centering
    \includegraphics[width=\textwidth]{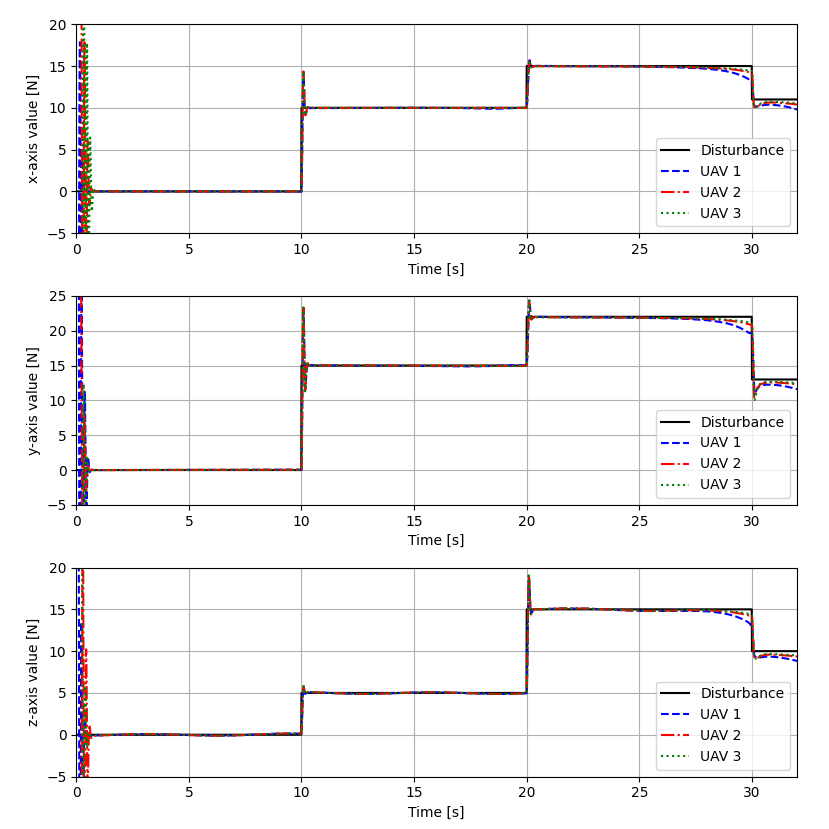}
    \caption{The estimated disturbances}
    \label{fig:est2}
    \end{subfigure}
    \caption{The tracking errors and estimated disturbances of the controllers in Scenario 2}
    \label{fig:scen2}
\end{figure}

\begin{figure}
\centering
\begin{subfigure}[b]{0.42\textwidth}
    \centering
    \frame{\includegraphics[width=\textwidth]{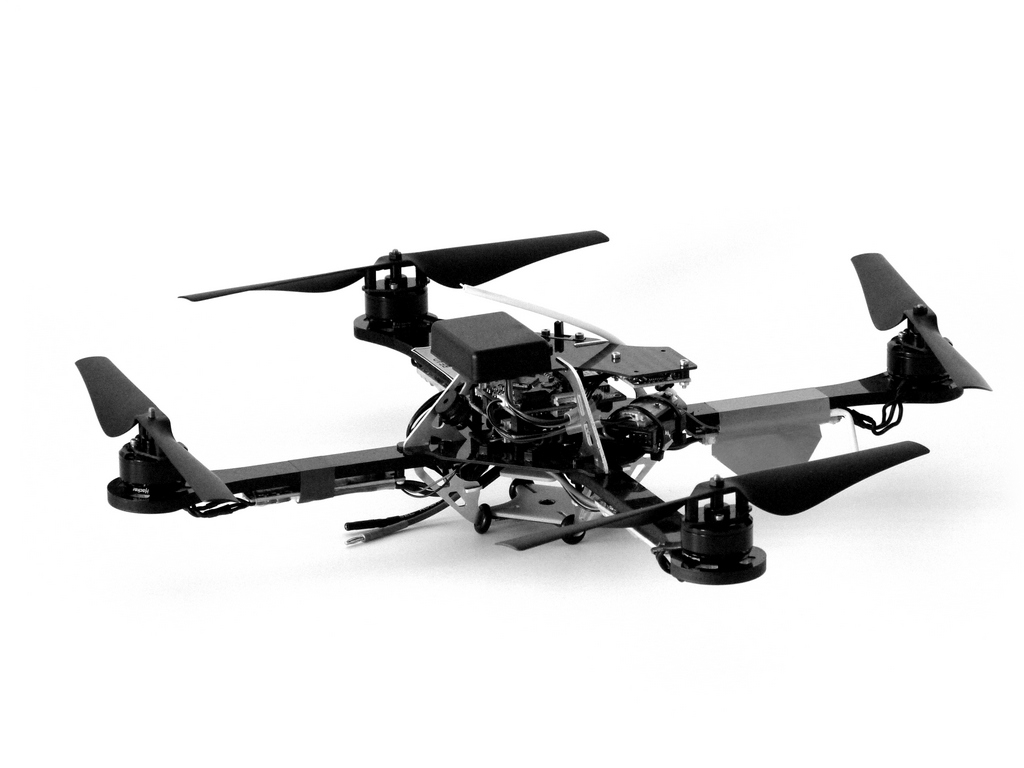}}
    \caption{The real Hummingbird drone}
    \label{fig:drone}
\end{subfigure}
\begin{subfigure}[b]{0.42\textwidth}
    \centering
    \includegraphics[width=\textwidth]{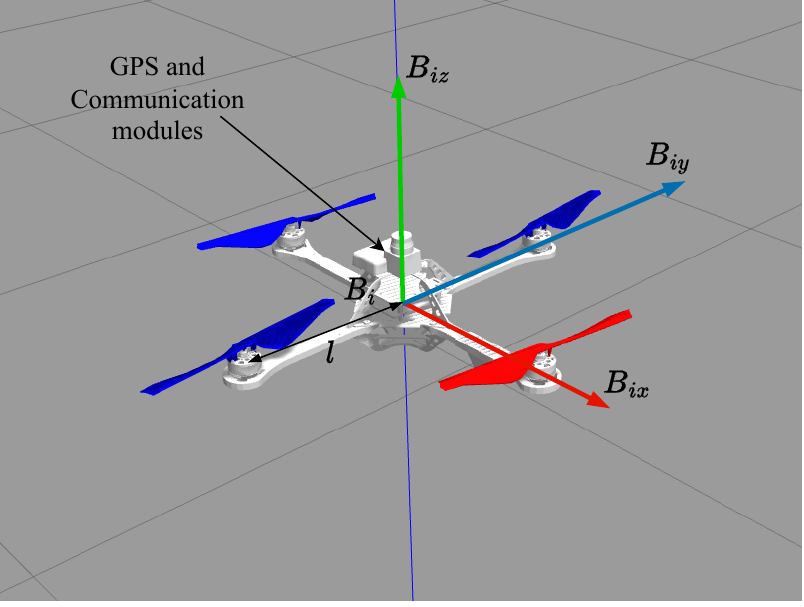}
    \caption{The Hummingbird model}
    \label{fig:used}
    % used as the reference model for SIL tests 
\end{subfigure}
\caption{The Hummingbird drone model used in validation \cite{Furrer2016,Bui2022}}
\label{fig:val}
\end{figure}

% \subsubsection{Bridge inspection results}
\begin{figure}
\centering
\begin{subfigure}[b]{0.42\textwidth}
    \centering
    \includegraphics[width=\textwidth]{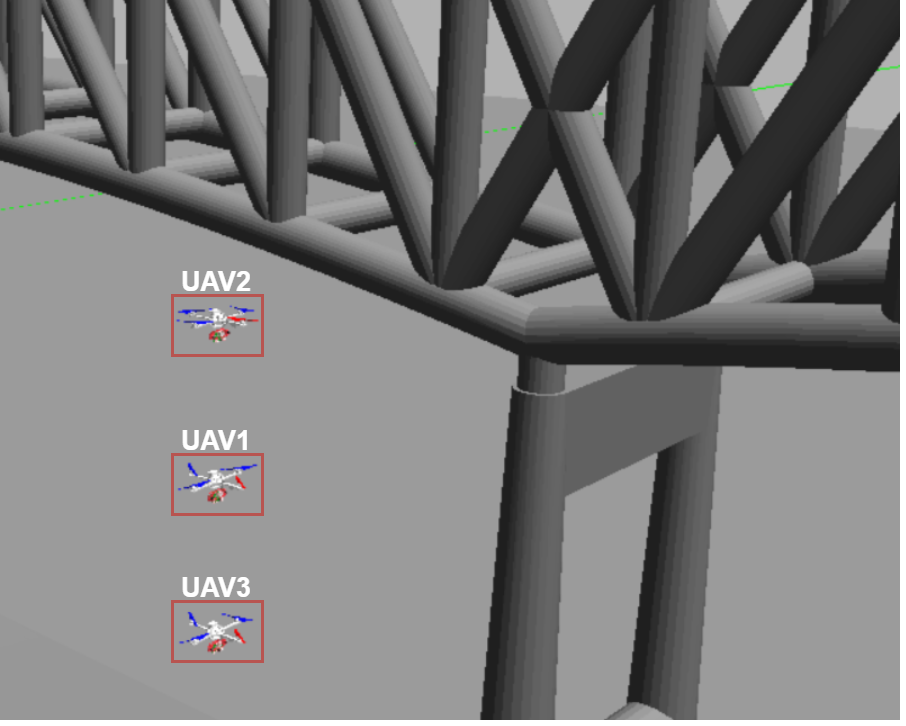}
    \caption{Vertical formation}
\end{subfigure}
\begin{subfigure}[b]{0.42\textwidth}
    \centering
    \includegraphics[width=\textwidth]{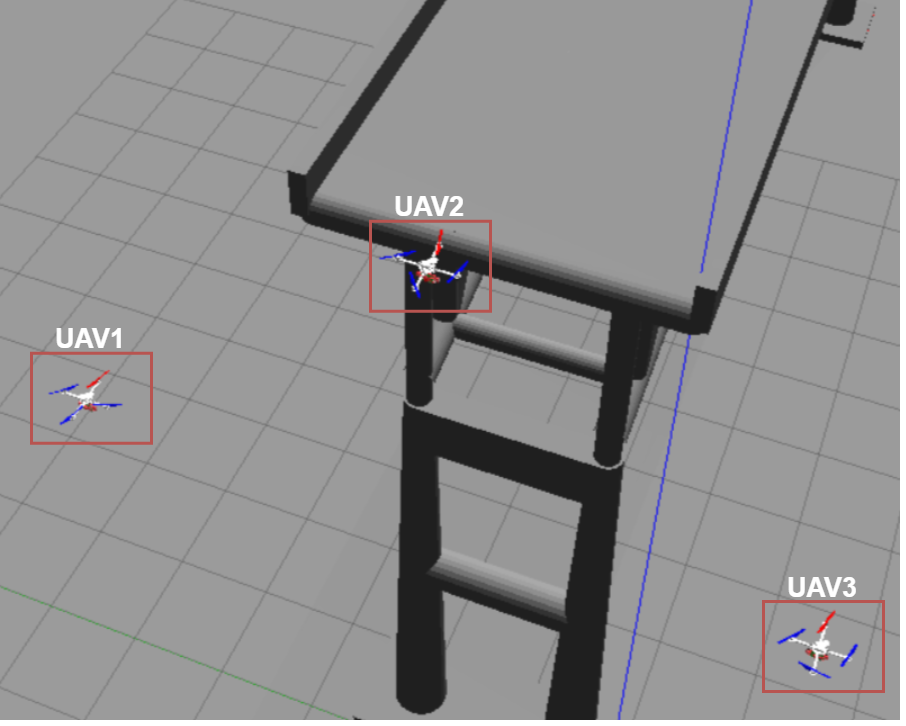}
    \caption{Triangular formation}
\end{subfigure}
\caption{The vertical and triangular formation topologies used in SIL tests}
\label{fig:val_form}
\end{figure}

\begin{figure}
\centering
    \begin{subfigure}[b]{0.9\textwidth}
    \centering
    \includegraphics[width=\textwidth]{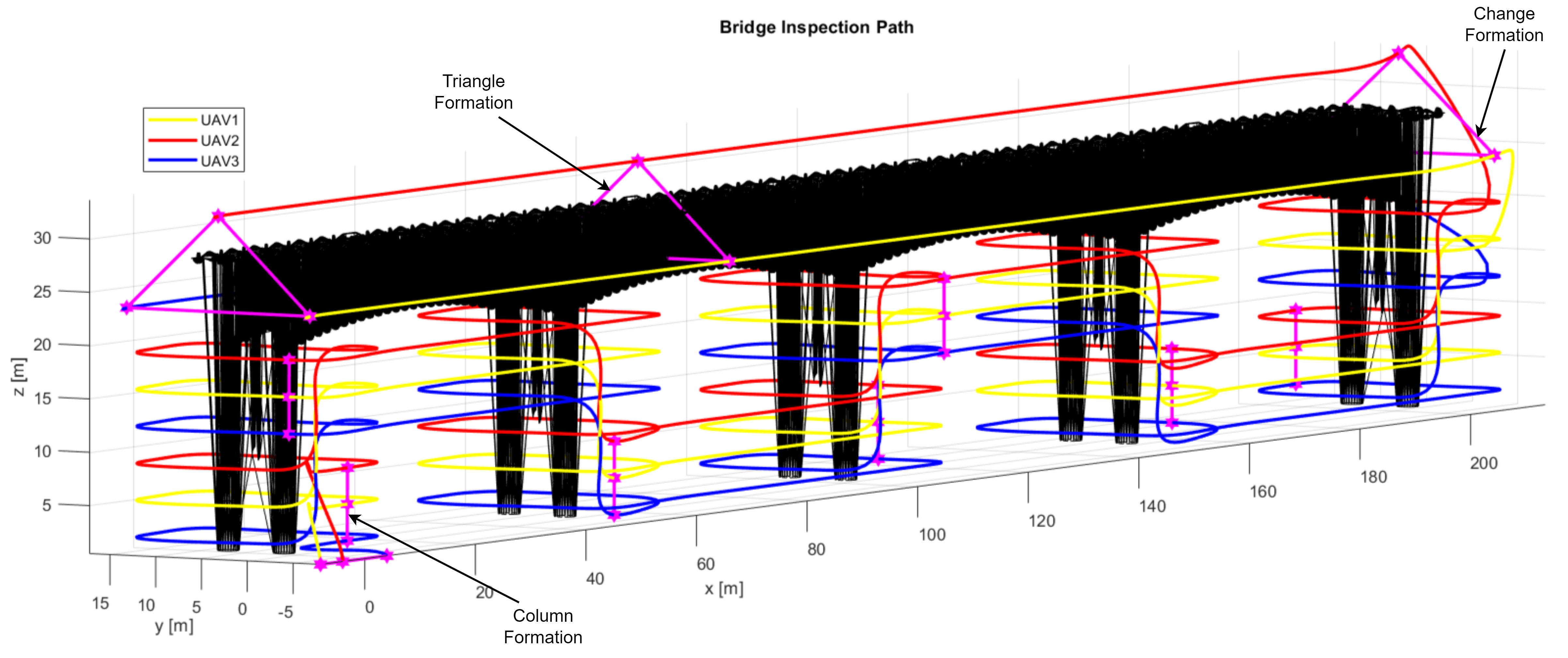}
    \caption{3D view}
    \label{fig:gazebo_side_a}
    \end{subfigure}
    \begin{subfigure}[b]{0.9\textwidth}
    \centering
    \includegraphics[width=\textwidth]{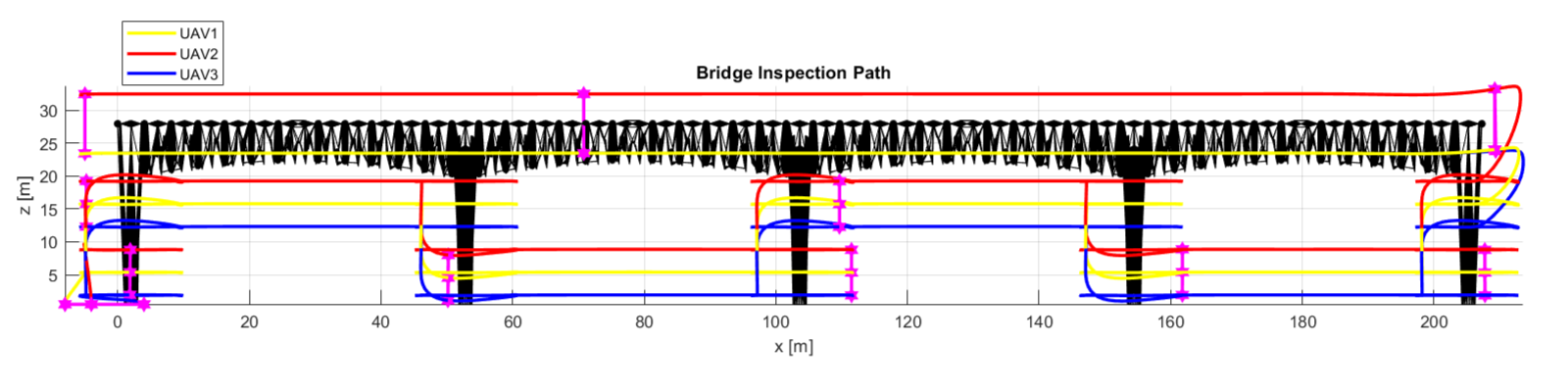}
    \caption{Side view}
    \label{fig:gazebo_side_b}
    \end{subfigure}
    \begin{subfigure}[b]{0.9\textwidth}
    \centering
    \includegraphics[width=\textwidth]{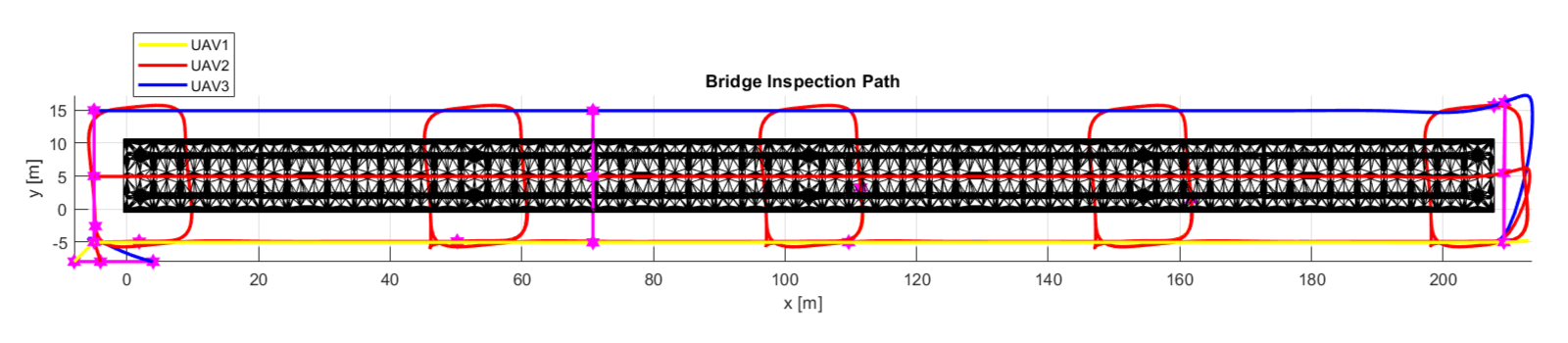}
    \caption{Top view}
    \label{fig:gazebo_side_c}
    \end{subfigure}
    \caption{The planned paths to inspect the bridge}
    \label{fig:gazebo_side}
\end{figure}

\begin{figure}
    \centering
    \includegraphics[width=0.9\textwidth]{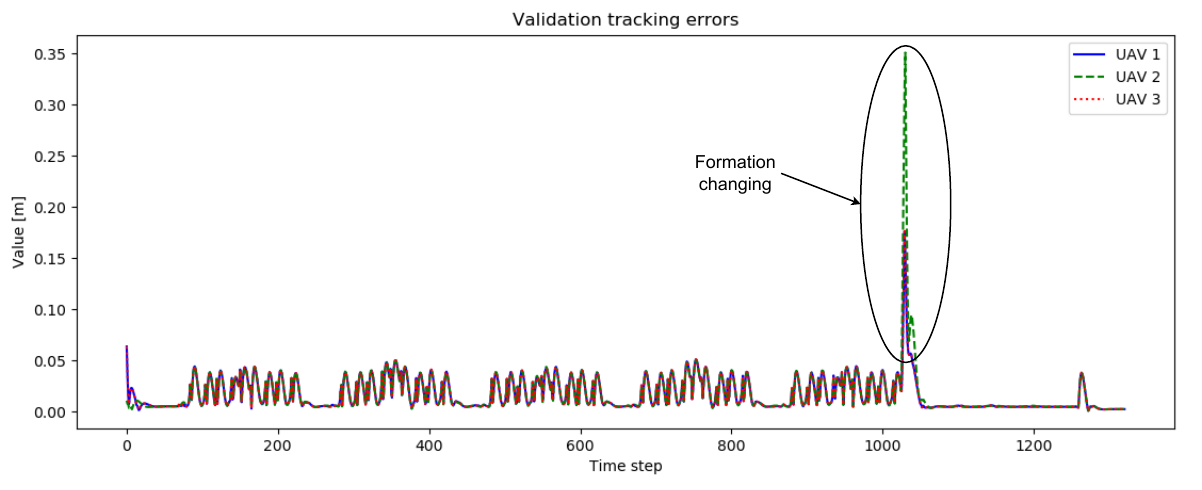}
    \caption{The tracking errors of UAV formation in the Gazebo SIL}
    \label{fig:sil_error}
\end{figure}

To evaluate the performance of the proposed control system, we have conducted a number of evaluations and comparisons\footnote{Evaluation results in Scenario 1 and Scenario 2 - \url{https://youtu.be/LYD7269n1-c}}. The UAV model used is the Hummingbird quadrotors \cite{Bui2022}, whose parameters are shown in Table \ref{tbl:uav_param}. Parameters of the position and attitude controllers are chosen as shown in Table \ref{tbl:control_param}. The desired formation is a triangular shape with $\Delta_1=\left[2,0,0\right]^T$, $\Delta_2=\left[0,0,2\right]^T$, and $\Delta_3=\left[0,0,-2\right]^T$, as depicted in Figure \ref{fig:topology}. Comparisons are conducted between the proposed controller (RBF-BSMC) and three other controllers namely model predictive control (MPC) \mbox{\cite{9329094,10064191}}, backstepping sliding mode control (BSMC) \mbox{\cite{9079172,8944049}} and sliding mode control (SMC) \mbox{\cite{4601469,9274320}} in different scenarios.

\subsection{Scenario 1}

In this scenario, external disturbances acting on the formation are generated based on the combination of the rectangle and full wavelength ``1-cosine'' wind model \cite{bo2019}, as shown in Figure \ref{fig:ex_dis}. The initial positions of the UAVs are set as $\xi_{1}=\left[3,2,4\right]^T$, $\xi_{2}=\left[2,1,4\right]^T$, and $\xi_{3}=\left[0,0,4\right]^T$. The desired trajectory of the virtual leader is a spiral with the $z$ coordinate increasing over time as expressed in (\ref{eq:spiral}).

\begin{equation}
    \begin{aligned}
    x_{L}&=5\cos\left(\dfrac{2\pi}{20}t\right)\text{ (m)},\\
    y_{L}&=5\sin\left(\dfrac{2\pi}{20}t\right)\text{ (m)},\\
    z_{L}&=0.5t+5\text{ (m)},\\
    \psi_{L}&=\dfrac{2\pi}{20}t+\dfrac{\pi}{2}\text{ (rad)}.
    \end{aligned}
    \label{eq:spiral}
\end{equation}

Figure \ref{fig:tracking} shows the 3D views of the trajectory tracking results of the UAV formation. It can be seen that the UAVs quickly reach the initial positions to form the desired shape. They then maintain the shape while following the reference trajectory. However, the trajectory of the proposed method is smoother and more accurate than the others due to its capability to estimate the disturbance via the RBFNN and use it as feedback to adjust the control signals. This result can be further verified via the tracking errors as shown in Figure \mbox{\ref{fig:error}}. It can be seen that the average tracking errors of the proposed controller quickly converge to zeros, whereas those errors of the other controllers largely fluctuate due to disturbances. In addition, the maximum and minimum tracking errors of the UAVs are also very small with our method, which confirm its stability for formation control.

Figure \ref{fig:est} shows the disturbances estimated for each UAV by the proposed controller. After the transition period, the estimation starts to converge to the real disturbance. This provides feedback for the controller to adjust the control signal for better tracking performance.

\subsection{Scenario 2}

In this scenario, the rectangle wind model is used to generate external disturbances as shown in Figure \ref{fig:ex_dis2}. The initial positions of the UAVs are set as $\xi_{1}=\left[2,0,9\right]^T$, $\xi_{2}=\left[5,2,7\right]^T$, and $\xi_{3}=\left[-3,-2,8\right]^T$. The desired trajectory of the virtual leader is generated based on an inspection path used to collect surface data \cite{PHUNG201725}.

The 3D of the UAV formation tracking results are shown in Figure \ref{fig:tracking2}. It can be seen that all controllers are able to drive the UAVs to reach their reference positions and then track them to form the desired shape. The proposed controller, however, introduces smaller tracking errors than the other controllers, as shown in Figure \ref{fig:error2}, due to its disturbance estimator. As shown in Figure \ref{fig:est2}, the estimation closely follows the actual disturbances except for the positions where step changes happen. At those positions, a transition period is needed for the estimator to converge to the new steady state. However, the settling time of the estimator is short allowing it to provide timely feedback to the controller.

\subsection{Validation with software-in-the-loop tests}

To further validate the proposed control system, we have carried out software-in-the-loop (SIL) tests that involve the inspection of a scaled-down 3D model of a real bridge with 5 columns, as shown in Figure \ref{fig:gazebo_side}. The UAV model used is a Hummingbird quadrotor\footnote{Source code used for Gazebo validation - \url{https://github.com/duynamrcv/hummingbird_simulator}} developed based on Gazebo-based RotorS simulator \cite{Furrer2016}, as depicted in Figure \ref{fig:val}. The formation used includes two topologies, vertical and triangular shapes, as shown in Figure \ref{fig:val_form}. According to our previous work \mbox{\cite{Bui2024}}, the generated path to inspect the bridge includes two stages. The first stage covers all columns of the bridge using a vertical formation. The second one uses a triangular formation to cover the side and top surfaces of the bridge, as depicted in Figure \ref{fig:gazebo_side}.

Given the planned paths, the UAVs start to fly from positions $\left[-8,-8,0\right]^T$, $\left[-4,-8,0\right]^T$, $\left[4,-8,0\right]^T$, and reach their initial positions to form a vertical formation as shown in Figure \ref{fig:gazebo_side_a}. The formation then tracks the planned path to acquire surface images of the bridge\footnote{SIL validation - \url{https://youtu.be/1yUCzWRDcp0}}. Figure \ref{fig:sil_error} shows the tracking errors of the UAVs during operation. It can be seen that the errors quickly converge to small values in both inspection stages except between time steps 1070 and 1120, where there is a change in the formation topology. The errors in the first stage are slightly larger than in the second one as the UAVs frequently changes their direction to navigate around each column of the bridge. Nevertheless, the average tracking error of less than 5 cm is sufficient for most UAV-related applications and thus confirms the validity of our approach.
\section{Conclusion}\label{con}
In this paper, we have presented a robust control system using RBFNN for a group of UAVs flying in a formation. By combining BSC with SMC, the controller can handle nonlinearity to increase its control performance. The use of RBFNN enables the system to estimate external disturbances to enhance its control robustness. By using Lyapunov's theorem, we proved that the control system is stable and the proposed controller can track the reference trajectory. Evaluation results show that the proposed controller outperforms the state-of-the-art BSMC in terms of accuracy and robustness and is sufficient for most UAV applications.

\begin{con}
\ctitle{Author Contributions}
Duy-Nam Bui: Conceptualization, Methodology, Implementation, Writing - original draft. Manh Duong Phung: Investigation, Conceptualization, Supervision, Writing - review and editing.

\ctitle{Financial Support}
Duy-Nam Bui was funded by the Master, PhD Scholarship Programme of Vingroup Innovation Foundation (VINIF), code VINIF.2022.ThS.057.

\ctitle{Conflicts of Interest}
The authors declare no conflicts of interest exist.
\end{con}

\bibliographystyle{ieeetr}
\bibliography{bibfile}

\begin{thebibliography}{10}

\bibitem{zeng_zhong2023}
J.~Zeng, H.~Zhong, Y.~Wang, S.~Fan, and H.~Zhang, ``Autonomous control design of an unmanned aerial manipulator for contact inspection,'' {\em Robotica}, vol.~41, no.~4, p.~1145–1158, 2023.

\bibitem{rizia_reyes2022}
M.~Rizia, J.~A. Reyes-Munoz, A.~G. Ortega, A.~Choudhuri, and A.~Flores-Abad, ``Autonomous aerial flight path inspection using advanced manufacturing techniques,'' {\em Robotica}, vol.~40, no.~7, p.~2128–2151, 2022.

\bibitem{INZERILLO2018457}
L.~Inzerillo, G.~{Di Mino}, and R.~Roberts, ``Image-based {{3D}} reconstruction using traditional and {{UAV}} datasets for analysis of road pavement distress,'' {\em Automation in Construction}, vol.~96, pp.~457--469, 2018.

\bibitem{ZHAO2021103832}
S.~Zhao, F.~Kang, J.~Li, and C.~Ma, ``Structural health monitoring and inspection of dams based on {{UAV}} photogrammetry with image {{3D}} reconstruction,'' {\em Automation in Construction}, vol.~130, p.~103832, 2021.

\bibitem{CHEN2021102913}
K.~Chen, G.~Reichard, X.~Xu, and A.~Akanmu, ``Automated crack segmentation in close-range building façade inspection images using deep learning techniques,'' {\em Journal of Building Engineering}, vol.~43, p.~102913, 2021.

\bibitem{PENG2021123896}
X.~Peng, X.~Zhong, C.~Zhao, A.~Chen, and T.~Zhang, ``A {UAV}-based machine vision method for bridge crack recognition and width quantification through hybrid feature learning,'' {\em Construction and Building Materials}, vol.~299, p.~123896, 2021.

\bibitem{la_dinh__2019}
H.~M. La, T.~H. Dinh, N.~H. Pham, Q.~P. Ha, and A.~Q. Pham, ``Automated robotic monitoring and inspection of steel structures and bridges,'' {\em Robotica}, vol.~37, no.~5, p.~947–967, 2019.

\bibitem{TIAN2022104043}
Y.~Tian, G.~Zhang, K.~Morimoto, and S.~Ma, ``Automated rust removal: Rust detection and visual servo control,'' {\em Automation in Construction}, vol.~134, p.~104043, 2022.

\bibitem{9341089}
W.~Jing, D.~Deng, Y.~Wu, and K.~Shimada, ``Multi-{UAV} coverage path planning for the inspection of large and complex structures,'' in {\em 2020 IEEE/RSJ International Conference on Intelligent Robots and Systems (IROS)}, pp.~1480--1486, 2020.

\bibitem{9384182}
G.~Silano, T.~Baca, R.~Penicka, D.~Liuzza, and M.~Saska, ``Power line inspection tasks with multi-aerial robot systems via signal temporal logic specifications,'' {\em IEEE Robotics and Automation Letters}, vol.~6, no.~2, pp.~4169--4176, 2021.

\bibitem{8593930}
V.~Hoang, M.~Phung, T.~Dinh, and Q.~Ha, ``Angle-encoded swarm optimization for {UAV} formation path planning,'' in {\em 2018 IEEE/RSJ International Conference on Intelligent Robots and Systems (IROS)}, pp.~5239--5244, 2018.

\bibitem{OH2015424}
K.-K. Oh, M.-C. Park, and H.-S. Ahn, ``A survey of multi-agent formation control,'' {\em Automatica}, vol.~53, pp.~424--440, 2015.

\bibitem{Liu2018}
Y.~Liu and R.~Bucknall, ``A survey of formation control and motion planning of multiple unmanned vehicles,'' {\em Robotica}, vol.~36, pp.~1019--1047, Mar. 2018.

\bibitem{8756125}
V.~T. Hoang, M.~D. Phung, T.~H. Dinh, and Q.~P. Ha, ``System architecture for real-time surface inspection using multiple {{UAV}s},'' {\em IEEE Systems Journal}, vol.~14, no.~2, pp.~2925--2936, 2020.

\bibitem{ZHENG2021389}
J.~Zheng, X.~Zong, H.~Ge, Z.~Zheng, and M.~C. Makuwatsine, ``Virtual leader-follower synchronization controller design for distributed parameter multi-agent systems with time-varying disturbances,'' {\em Neurocomputing}, vol.~450, pp.~389--398, 2021.

\bibitem{Rinaldi2013}
F.~Rinaldi, S.~Chiesa, and F.~Quagliotti, ``Linear quadratic control for quadrotors {UAV}s dynamics and formation flight,'' {\em Journal of Intelligent {\&} Robotic Systems}, vol.~70, pp.~203--220, Apr 2013.

\bibitem{doi:10.1177/0278364909104290}
J.~Chen, D.~Sun, J.~Yang, and H.~Chen, ``Leader-follower formation control of multiple non-holonomic mobile robots incorporating a receding-horizon scheme,'' {\em The International Journal of Robotics Research}, vol.~29, no.~6, pp.~727--747, 2010.

\bibitem{9908553}
B.-S. Chen, Y.-C. Liu, M.-Y. Lee, and C.-L. Hwang, ``Decentralized h {PID} team formation tracking control of large-scale quadrotor {UAV}s under external disturbance and vortex coupling,'' {\em IEEE Access}, vol.~10, pp.~108169--108184, 2022.

\bibitem{Chen2023}
Y.~Chen and T.~Deng, ``Leader-follower {UAV} formation flight control based on feature modelling,'' {\em Systems Science \& Control Engineering}, vol.~11, Oct. 2023.

\bibitem{Fahimi2008}
F.~Fahimi, ``Full formation control for autonomous helicopter groups,'' {\em Robotica}, vol.~26, pp.~143--156, Mar. 2008.

\bibitem{KeymasiKhalaji2019}
A.~K. Khalaji and R.~Zahedifar, ``Lyapunov-based formation control of underwater robots,'' {\em Robotica}, vol.~38, pp.~1105--1122, Aug. 2019.

\bibitem{DEHGHANI2016318}
M.~A. Dehghani and M.~B. Menhaj, ``Integral sliding mode formation control of fixed-wing unmanned aircraft using seeker as a relative measurement system,'' {\em Aerospace Science and Technology}, vol.~58, pp.~318--327, 2016.

\bibitem{4601469}
M.~Defoort, T.~Floquet, A.~Kokosy, and W.~Perruquetti, ``Sliding-mode formation control for cooperative autonomous mobile robots,'' {\em IEEE Transactions on Industrial Electronics}, vol.~55, no.~11, pp.~3944--3953, 2008.

\bibitem{7904763}
R.~Li, L.~Zhang, L.~Han, and J.~Wang, ``Multiple vehicle formation control based on robust adaptive control algorithm,'' {\em IEEE Intelligent Transportation Systems Magazine}, vol.~9, no.~2, pp.~41--51, 2017.

\bibitem{HUANG20204034}
Y.~Huang, W.~Liu, B.~Li, Y.~Yang, and B.~Xiao, ``Finite-time formation tracking control with collision avoidance for quadrotor {{UAV}s},'' {\em Journal of the Franklin Institute}, vol.~357, no.~7, pp.~4034--4058, 2020.

\bibitem{9707476}
X.~Wang, S.~Baldi, X.~Feng, C.~Wu, H.~Xie, and B.~De~Schutter, ``A fixed-wing {UAV} formation algorithm based on vector field guidance,'' {\em IEEE Transactions on Automation Science and Engineering}, vol.~20, no.~1, pp.~179--192, 2023.

\bibitem{10064191}
Q.~Yuan and X.~Li, ``Distributed model predictive formation control for a group of {UAV}s with spatial kinematics and unidirectional data transmissions,'' {\em IEEE Transactions on Network Science and Engineering}, vol.~10, no.~6, pp.~3209--3222, 2023.

\bibitem{YANG2021243}
S.~Yang, W.~Bai, T.~Li, Q.~Shi, Y.~Yang, Y.~Wu, and C.~L.~P. Chen, ``Neural-network-based formation control with collision, obstacle avoidance and connectivity maintenance for a class of second-order nonlinear multi-agent systems,'' {\em Neurocomputing}, vol.~439, pp.~243--255, 2021.

\bibitem{SHOJAEI2016372}
K.~Shojaei, ``Neural network formation control of {U}nderactuated {A}utonomous {U}nderwater vehicles with saturating actuators,'' {\em Neurocomputing}, vol.~194, pp.~372--384, 2016.

\bibitem{8651430}
C.-W. Kuo, C.-C. Tsai, and C.-T. Lee, ``Intelligent leader-following consensus formation control using recurrent neural networks for small-size unmanned helicopters,'' {\em IEEE Transactions on Systems, Man, and Cybernetics: Systems}, vol.~51, no.~2, pp.~1288--1301, 2021.

\bibitem{6796390}
E.~J. Hartman, J.~D. Keeler, and J.~M. Kowalski, ``Layered neural networks with gaussian hidden units as universal approximations,'' {\em Neural Computation}, vol.~2, no.~2, pp.~210--215, 1990.

\bibitem{Furrer2016}
F.~Furrer, M.~Burri, M.~Achtelik, and R.~Siegwart, {\em Robot Operating System (ROS): The Complete Reference (Volume 1)}, ch.~RotorS---A Modular Gazebo MAV Simulator Framework, pp.~595--625.
\newblock Cham: Springer International Publishing, 2016.

\bibitem{NARENDRA1987}
K.~S. NARENDRA and A.~M. ANNASWAMY, ``Persistent excitation in adaptive systems,'' {\em International Journal of Control}, vol.~45, pp.~127--160, Jan. 1987.

\bibitem{Bui2022}
D.~N. Bui, T.~T. Van~Nguyen, and M.~D. Phung, ``Lyapunov-based nonlinear model predictive control for attitude trajectory tracking of unmanned aerial vehicles,'' {\em International Journal of Aeronautical and Space Sciences}, Oct 2022.

\bibitem{9329094}
D.~Wang, Q.~Pan, Y.~Shi, J.~Hu, and C.~Zhao, ``Efficient nonlinear model predictive control for quadrotor trajectory tracking: Algorithms and experiment,'' {\em IEEE Transactions on Cybernetics}, vol.~51, no.~10, pp.~5057--5068, 2021.

\bibitem{9079172}
L.-X. Xu, H.-J. Ma, D.~Guo, A.-H. Xie, and D.-L. Song, ``Backstepping sliding-mode and cascade active disturbance rejection control for a quadrotor {UAV},'' {\em IEEE/ASME Transactions on Mechatronics}, vol.~25, no.~6, pp.~2743--2753, 2020.

\bibitem{8944049}
D.~J. Almakhles, ``Robust backstepping sliding mode control for a quadrotor trajectory tracking application,'' {\em IEEE Access}, vol.~8, pp.~5515--5525, 2020.

\bibitem{9274320}
I.~Ahmad, M.~Liaquat, F.~M. Malik, H.~Ullah, and U.~Ali, ``Variants of the sliding mode control in presence of external disturbance for quadrotor,'' {\em IEEE Access}, vol.~8, pp.~227810--227824, 2020.

\bibitem{bo2019}
B.~H. Wang, D.~B. Wang, Z.~A. Ali, B.~T. Ting, and H.~Wang, ``An overview of various kinds of wind effects on unmanned aerial vehicle,'' {\em Measurement and Control}, vol.~52, no.~7-8, pp.~731--739, 2019.

\bibitem{PHUNG201725}
M.~D. Phung, C.~H. Quach, T.~H. Dinh, and Q.~Ha, ``Enhanced discrete particle swarm optimization path planning for {UAV} vision-based surface inspection,'' {\em Automation in Construction}, vol.~81, pp.~25--33, 2017.

\bibitem{Bui2024}
D.~N. Bui, T.~N. Duong, and M.~D. Phung, ``Ant colony optimization for cooperative inspection path planning using multiple unmanned aerial vehicles,'' in {\em 2024 IEEE/SICE International Symposium on System Integration (SII)}, pp.~675--680, 2024.

\end{thebibliography}
\end{document}